\documentclass[journal]{IEEEtran}
\usepackage{amsmath,amsfonts}
\usepackage{array}
\usepackage[caption=false,font=normalsize,labelfont=sf,textfont=sf]{subfig}
\usepackage{textcomp}
\usepackage{stfloats}
\usepackage{url}
\usepackage{verbatim}
\usepackage{graphicx}
\usepackage{cite}
\usepackage{xcolor}
\usepackage{colortbl}
\definecolor{baselinecolor}{gray}{.9}
\newcommand{\baseline}[1]{\cellcolor{baselinecolor}\makebox[0pt][l]{#1}}

\usepackage{threeparttable}
\usepackage{mathtools}
\usepackage{amsmath,amssymb,amsfonts}
\usepackage{amsthm}
\usepackage{bbm}
\usepackage{bm}
\usepackage{bbding} 
\usepackage{multirow}
\usepackage{booktabs}  
\usepackage{array}
\usepackage{amsfonts}
\usepackage{hyperref}
\usepackage{float}
\usepackage{bbm}
\usepackage{cases}
\usepackage{graphicx}
\usepackage{capt-of}
\usepackage[linesnumbered,ruled,vlined]{algorithm2e}
\usepackage{varwidth}
\newtheorem{prop}{Proposition}
\newcommand\identity{1\kern-0.25em\text{l}}
\graphicspath{ {./figs/} }

\begin{document}

\title{Spatio-Temporal Decoupled Learning for \\ Spiking Neural Networks}

\author{
	Chenxiang~Ma, 	
    Xinyi~Chen,
    Kay~Chen~Tan,~\IEEEmembership{Fellow,~IEEE}
	Jibin~Wu,~\IEEEmembership{Member,~IEEE}

\thanks{
C.~Ma, and X.~Chen are with the Department of Data Science and Artificial Intelligence, The Hong Kong Polytechnic University, Hong Kong SAR.}%

\thanks{
K.C.~Tan is with the Department of Data Science and Artificial Intelligence and the Research Center of Data Science and Artificial Intelligence, The Hong Kong Polytechnic University, Hong Kong, SAR.}%
\thanks{
J.~Wu is with the Department of Data Science and Artificial Intelligence, the Department of Computing, and the Research Center of Data Science and Artificial Intelligence, The Hong Kong Polytechnic University, Hong Kong, SAR. \textit{Corresponding author: Jibin Wu (e-mail: jibin.wu@polyu.edu.hk).} }%
}


\maketitle

\begin{abstract}
Spiking neural networks (SNNs) have gained significant attention for their potential to enable energy-efficient artificial intelligence. However, effective and efficient training of SNNs remains an unresolved challenge. While backpropagation through time (BPTT) achieves high accuracy, it incurs substantial memory overhead. In contrast, biologically plausible local learning methods are more memory-efficient but struggle to match the accuracy of BPTT. To bridge this gap, we propose spatio-temporal decouple learning (STDL), a novel training framework that decouples the spatial and temporal dependencies to achieve both high accuracy and training efficiency for SNNs. Specifically, to achieve spatial decoupling, STDL partitions the network into smaller subnetworks, each of which is trained independently using an auxiliary network. To address the decreased synergy among subnetworks resulting from spatial decoupling, STDL constructs each subnetwork’s auxiliary network by selecting the largest subset of layers from its subsequent network layers under a memory constraint. Furthermore, STDL decouples dependencies across time steps to enable efficient online learning. Extensive evaluations on seven static and event-based vision datasets demonstrate that STDL consistently outperforms local learning methods and achieves comparable accuracy to the BPTT method with considerably reduced GPU memory cost. Notably, STDL achieves 4$\times$ reduced GPU memory than BPTT on the ImageNet dataset. Therefore, this work opens up a promising avenue for memory-efficient SNN training. Code is available at \url{https://github.com/ChenxiangMA/STDL}.
\end{abstract}


\section{Introduction}
\IEEEPARstart{H}{uman} brain is renowned for its remarkable efficiency and versatility in cognitive computing, which has long captivated researchers to unravel its underlying structure and operating mechanisms~\cite{dayan2001theoretical,berwick2013evolution}. These insights have, in turn, driven the development of brain-inspired artificial intelligence~(AI) systems aimed at replicating and harnessing the brain's remarkable information processing capabilities~\cite{richards2019deep,zador2023catalyzing}.
Among these efforts, spiking neural networks~(SNNs) have emerged as a promising approach~\cite{maass1997networks}. SNNs seek to faithfully model the dynamic behaviors of biological neural networks, leveraging the spike-based information processing principles observed in the brain~\cite{roy2019towards}. 
Unlike traditional artificial neural networks (ANNs) that use analog values~\cite{lecun2015deep}, SNNs employ binary spikes for information representation. Each spiking neuron has a membrane potential, an internal state that evolves and processes incoming spikes over time. This unique spike-based representation, coupled with rich spatiotemporal neuronal dynamics, enables SNNs to perform efficient event-driven computation, where neurons remain inactive until they receive incoming  spikes~\cite{roy2019towards}. 
Characterized by such event-driven computation as well as inherent sparsity, SNNs can offer remarkable energy efficiency when deployed on neuromorphic chips~\cite{spikingfullsubnet,davies2018loihi,pei2019towards}, thus serving as a compelling solution for energy-efficient AI~\cite{li2023brain,eshraghian2023training}.

However, the training of SNNs is non-trivial. Due to the dynamic nature of SNNs, the backpropagation through time~(BPTT)~\cite{werbos1990backpropagation,wu2018spatio} algorithm is commonly adopted for training SNNs. BPTT effectively computes parameter gradients by unfolding the SNN over time and backpropagating global error signals through both network layers and time steps. One major challenge in training SNNs lies in the non-differentiability of the neuronal activation function. This is often solved by using surrogate gradients~\cite{neftci2019surrogate,lian2023learnable}, i.e., replacing the gradient of the non-differentiable activation function with a continuous surrogate function that approximates the true gradient. Furthermore, for BPTT training, gradients cannot be computed until forward propagation across all layers and time steps is completed. 
Consequently, the intermediate neuronal states must be stored in memory until they are revisited for gradient computation, leading to high memory overhead.
This becomes a critical bottleneck when simulating large-scale SNN architectures~\cite{zenke2020brain}.

The biologically inspired local learning approach is an appealing alternative to circumvent the memory overhead issue in BPTT by breaking spatial and temporal dependencies~\cite{ma2023ell,ma2024scaling}. It enables the online training of each layer independently, and only requires storing the states for a single layer at each time step, leading to significantly reduced memory usage.
The unsupervised spike-timing-dependent plasticity~(STDP)~\cite{bi1998synaptic,diehl2015unsupervised,kheradpisheh2018stdp,hao2020biologically,2018zhengonline,2025hebbiansnn} is the most widely adopted local learning rule for SNNs, which modulates synaptic weights based on the relative timings of pre- and post-synaptic spikes. However, SNNs trained with STDP are limited to performing simple tasks~\cite{apolinario2025stllr}. Recently, studies have applied gradient descent to local learning by optimizing each layer with an auxiliary classifier and a supervised local loss function~\cite{ma2023ell,kaiser2020synaptic}.
These supervised local learning methods achieve significantly improved accuracy over STDP. While they suffer from poor scalability, they still represent an important step forward in achieving high accuracy with reduced memory requirements compared to BPTT (see Fig.~\ref{fig:fig1}). Therefore, achieving both high accuracy and memory efficiency in training SNNs remains a critical challenge that requires further research.
\begin{figure}[!t]
\centering\includegraphics[width=0.85\linewidth]{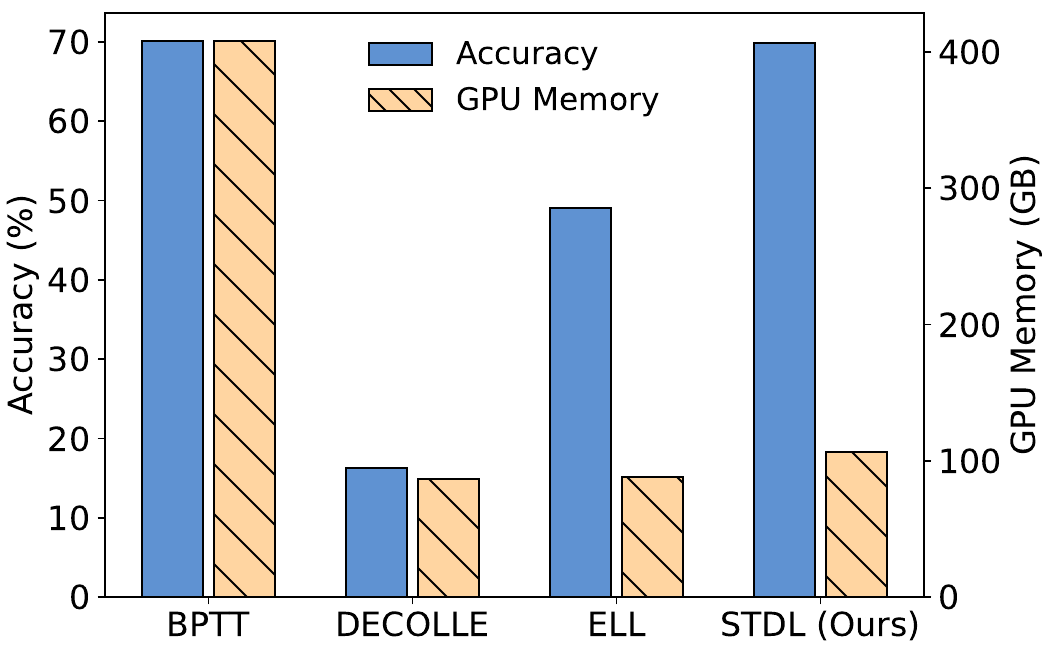}
\caption{Comparison of BPTT, supervised local learning rules (DECOLLE and ELL), and the proposed STDL method in terms of accuracy and GPU memory for training the SEWResNet-34 network on the ImageNet dataset.}
\label{fig:fig1}
\end{figure}

In this work, we begin by analyzing supervised local learning methods for training SNNs, and identify that the decoupling of spatial dependencies in these approaches often leads to a weak coupling issue. Specifically, the locally trained layers, optimized solely with their individual objectives, fail to synergize with subsequent layers in achieving the global objective. This represents a major bottleneck that limits the overall network performance. See Section~\ref{sec:mem_bptt} for more details.

To address this limitation, we propose a novel spatio-temporal decoupled learning (STDL) framework that enables both high-accuracy and memory-efficient training of SNNs. As illustrated in Fig.~\ref{Fig:BPTTvsSTDL}(c), STDL partitions the full network into several subnetworks, each of which is independently trained with an associated auxiliary network, without receiving gradients from subsequent layers. However, increasing the number of subnetworks exacerbates the weak coupling issue. To overcome this, we propose a greedy network partitioning algorithm that constructs a minimal number of subnetworks under a given memory constraint. This is reasonable since layers in typical SNN architectures~\cite{fang2021deep,simonyan2014very,he2016deep} have different memory requirements, allowing them to be grouped into a single subnetwork without affecting the peak memory usage. Moreover, we theoretically prove that the proposed greedy algorithm yields the optimal subnetwork configuration. 

To further enhance the coupling among subnetworks, we design each auxiliary network to guide its associated subnetwork towards producing representations that closely align with those learned via BPTT. In addition, our information-theoretic analysis suggests that enhancing the capacity of the auxiliary network improves the representation alignment. Motivated by these insights, we propose a principled method for constructing auxiliary networks. Specifically, for each subnetwork, we select a subset of its subsequent layers to form an auxiliary network that has the largest capacity while respecting the memory constraint. By combining this construction method with the greedy network partitioning algorithm, STDL enables subnetworks to coordinate effectively toward the global objective, thereby addressing the weak coupling issue. Furthermore, STDL decouples temporal dependencies in the training of each subnetwork by neglecting temporally dependent gradients that contribute negligibly to the overall parameter updates. This ensures that gradients are immediately available at each time step, further improving memory efficiency.

We conduct extensive experiments on image classification datasets, including CIFAR-10~\cite{krizhevsky2009learning}, CIFAR-100~\cite{krizhevsky2009learning}, SVHN~\cite{netzer2011reading}, and ImageNet~\cite{deng2009imagenet}, as well as event-based vision datasets, including DVS-CIFAR-10~\cite{li2017CIFAR10}, GAIT-DAY~\cite{gait_day_pami}, and HAR-DVS~\cite{2024_hardvs}. The results consistently demonstrate the superior performance of STDL across these varied task domains. 
For instance, on the ImageNet dataset using the SEWResNet34~\cite{fang2021deep} network, STDL outperforms the state-of-the-art supervised local learning methods, deep continuous local learning (DECOLLE)~\cite{kaiser2020synaptic} and efficient local learning~(ELL)~\cite{ma2023ell}, by over 50\% and 20\% in accuracy, respectively. This significant performance improvement highlights the effectiveness of STDL in enhancing the coordination between the independently trained subnetworks. Notably, STDL achieves comparable accuracy to the BPTT method while reducing the GPU memory usage by a factor of $4.0\times$. 

Our key contributions are summarized as follows:
\begin{itemize}
    \item STDL is the first method that effectively harnesses the memory advantage of local learning while achieving performance on par with BPTT, representing a significant step toward memory-efficient SNN training.
    \item We propose novel methods for partitioning the full network into subnetworks and constructing auxiliary networks. These novel designs effectively address the weak coupling issue, delivering enhanced performance while retaining the memory efficiency advantage.
    \item We conduct a comprehensive validation of STDL on seven benchmark datasets, consistently demonstrating remarkable accuracy that is significantly superior to existing local learning rules and on par with BPTT, as well as the substantially reduced GPU memory usage. Our experiments also demonstrate the strong generalization ability of STDL to different spiking neuron models.    
    \item We conduct extensive ablation studies to evaluate the memory efficiency of STDL and to assess the contribution of each individual component. In addition, our representation analysis shows that STDL learns representations closely aligned with those produced by BPTT, offering a compelling explanation for the efficacy of STDL.
\end{itemize}

The remainder of this work is organized as follows. 
In Section~\ref{sec:related_works}, we provide a review of existing SNN training methods.
In Section~\ref{sec:mem_bptt}, we analyze the memory demands of BPTT as well as the reasons behind the inferior accuracy of existing supervised local learning methods, and identify the weak coupling issue as the major bottleneck.
Section~\ref{sec:STDL} details the proposed STDL method, followed by extensive experimental results and analyses in Section~\ref{sec:exp_res}. Finally, we conclude the paper in Section~\ref{sec:conclu}.

\begin{figure}[!tb]
\centering\includegraphics[width=0.9\linewidth]{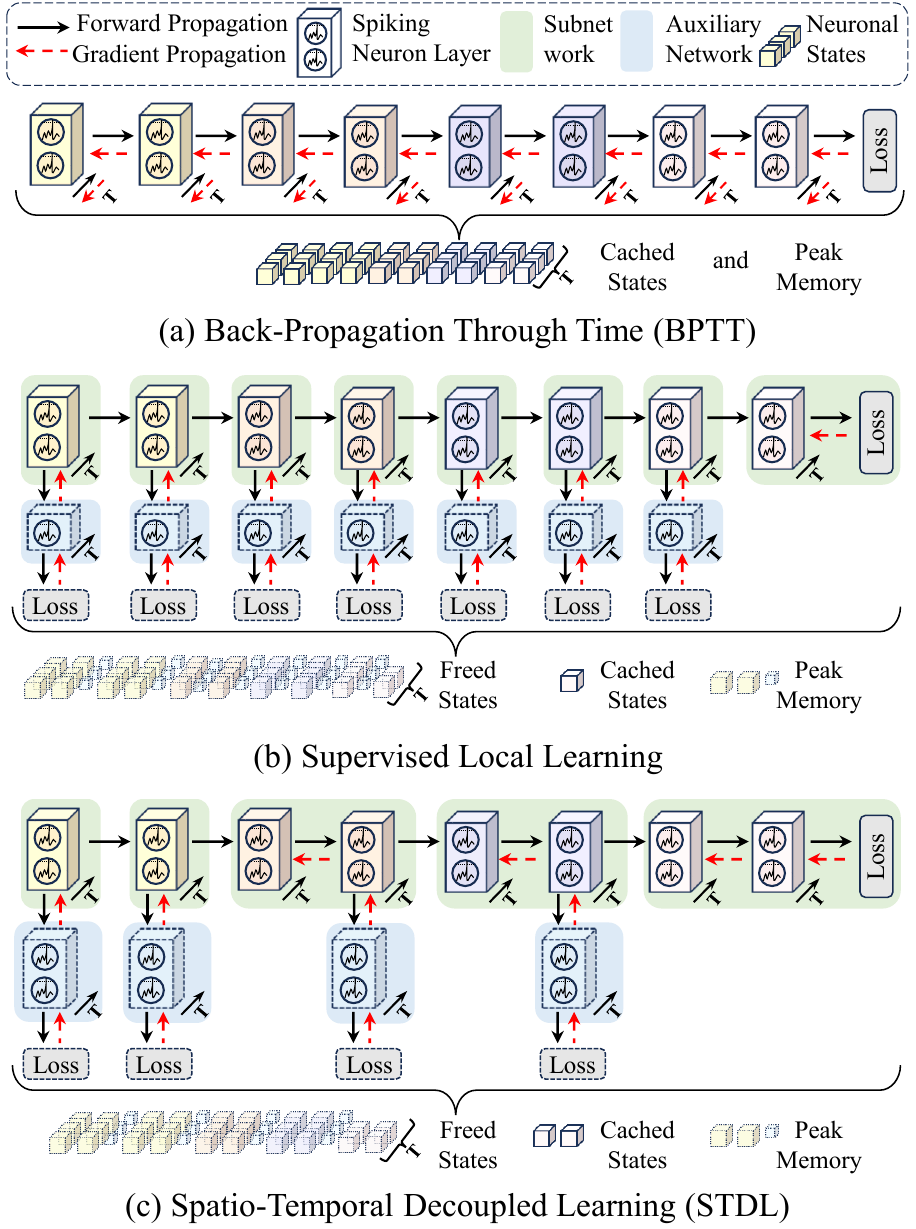}
\caption{\textbf{Comparison of BPTT, supervised local learning, and our STDL for training SNNs.}
(a) BPTT caches the entire trajectory of neuronal states during forward propagation for gradient computation, with memory overhead scaled along with both the number of layers and time steps. (b) Supervised local learning caches the states of only one layer at each time step, as it enables the online learning of each network layer along with an auxiliary linear classifier independently. (c) STDL partitions the original network into subnetworks by grouping as many adjacent layers as possible without affecting the peak memory. Each subnetwork is independently trained with an auxiliary network in an online learning manner. For both (b) and  (c), once the forward propagation of a subnetwork and its auxiliary network is completed, their gradients are immediately calculated, followed by freeing the cached states. }
\label{Fig:BPTTvsSTDL}
\end{figure}

\section{Related Works}
\label{sec:related_works}
In this section, we briefly review existing training methods for SNNs.

\subsection{BPTT Training}
Recently, BPTT training for SNNs has been advanced by developing adaptive surrogate gradients~\cite{li2021differentiable,wang2023adaptive}.
In addition, some efforts have been devoted to improving the stability of BPTT training.
For instance, novel loss functions~\cite{deng2021temporal,guo2022recdis} are developed to regulate the distribution of membrane potential, facilitating network optimization and convergence. 
Specialized variants of batch normalization techniques~\cite{zheng2021going,duan2022temporal} are developed for SNNs to accelerate convergence. Furthermore, trainable decays~\cite{fang2021incorporating}, adaptive thresholds~\cite{yin2021accurate}, gating mechanisms~\cite{yao2022glif}, attention mechanisms~\cite{yao2023attentionsnn,2024attentionsnn}, hybrid coding~\cite{2023chenhybridcoding}, augmented spikes~\cite{2022auglearn,2022AugMapping}, and multi-compartmental modeling~\cite{zhang2023tc,chen2024a} are explored to augment the capabilities of spiking neuron models, thereby improving BPTT's performance on practical tasks. 
However, BPTT incurs substantial GPU memory overhead.

\subsection{Local Learning}
Local learning has gained increasing attention due to its appealing characteristics of biological plausibility and memory efficiency~\cite{ma2024scaling,ororbia2023brain}. 
An early attempt~\cite{diehl2015unsupervised} to leverage the STDP rule to train single-layer spiking networks effectively enables unsupervised visual feature extraction. Subsequent works extend the approach to train spiking convolutional networks~\cite{kheradpisheh2018stdp} or utilize a symmetric variant of STDP in combination with other bio-plausible mechanisms to improve performance~\cite{hao2020biologically}. Despite these advancements, STDP-trained SNNs demonstrate effectiveness primarily on simple tasks, such as MNIST~\cite{kheradpisheh2018stdp,hao2020biologically}, and encounter significant challenges when dealing with more complex tasks. 
Regarding supervised local learning, DECOLLE~\cite{kaiser2020synaptic} employs fixed and random auxiliary classifiers to construct layer-wise loss functions. It also enables continuous training at each time step by ignoring all temporally dependent gradients. ELL~\cite{ma2023ell} further enhances accuracy by replacing the fixed auxiliary classifiers with trainable ones. 
However, they still have a large accuracy gap compared to BPTT, especially for large-scale SNNs. 

In addition, previous efforts have developed hybrid local-global learning methods. In local tandem learning~(LTL)~\cite{2022ltl}, local learning is employed to train SNNs assisted by a pre-trained ANN. Local plasticity can also be integrated with top-down supervision to facilitate multi-scale learning~\cite{wu2022brain}. A recent work proposes the spiking time sparse feedback (STSF) method~\cite{STSF} that associates sparse direct feedback alignment with STDP to achieve hybrid learning.

\subsection{Online Learning}
Online learning can be regarded as a simplified variant of local learning with decoupled temporal dependencies only, where parameter gradients are computed in real time.
Recent works leverage gradient approximation to achieve the online learning of SNNs. For instance, E-Prop~\cite{bellec2020solution} and OTTT~\cite{xiao2022online} approximate the gradients of BPTT by eligibility traces, which are the temporal accumulation of presynaptic activities. 
SLTT~\cite{meng2023towards} disregards all temporal dependencies in gradient computation, with no need for maintaining additional traces.
Another work~\cite{yin2023accurate} proposes a specific neuron model and a regularization term to improve online learning, which introduces additional computational overhead, limiting its application on large architectures. Despite these advancements,  online learning methods require storing states of all layers, leading to substantial memory usage.

\subsection{Other End-to-End Learning}
Other approaches for end-to-end training of SNNs include the ANN-to-SNN conversion approach~\cite{cao2015spiking,2022AugMapping,2025towardwang}, which constructs an SNN by converting weights from a pre-trained ANN with a similar structure. Additionally, tandem learning approaches~\cite{2023tandemwu,2022PTLwu,2022ltl} introduce weight sharing between an ANN and an SNN, where the ANN provides gradients to train the SNN. These methods fundamentally differ from ours as they rely on global objectives.

\section{Motivation}
\label{sec:mem_bptt}
In this section, we first analyze the substantial memory overhead associated with training an SNN using BPTT~\cite{wu2018spatio}. We then elucidate how local learning can alleviate this memory bottleneck. Finally, we examine the underlying reasons for the significant performance gap between existing local learning methods and BPTT.

\subsection{Memory Consumption of BPTT}
During the forward propagation, input samples traverse the network layer-by-layer until they reach the output layer, where predictions are generated.
Let us consider a layer $\ell$ which receives afferent spikes from the preceding layer.  
These input spikes are integrated into an internal state called the membrane potential. 
When the membrane potential exceeds a pre-defined threshold, a spike is generated, followed by a reset process. The neuronal dynamics of the commonly-used leaky integrate-and-fire (LIF) neuron model can be represented in the following discrete-time form~\cite{wu2018spatio}:
\begin{align}
\boldsymbol{m}^{\ell}[t] &= \lambda \boldsymbol{u}^{\ell}[t - 1]  + \boldsymbol{W}^{\ell} \boldsymbol{s}^{\ell-1}[t], \\
\boldsymbol{s}^{\ell}[t]  &= \Theta (\boldsymbol{m}^{\ell}[t] - V_{\mathrm{th}}), \\
    \boldsymbol{u}^{\ell}[t] &= \boldsymbol{m}^{\ell}[t] - V_{\mathrm{th}}\boldsymbol{s}^{\ell}[t], \\
        \Theta(x)&=\begin{cases}1, & x >= 0\\0, & \text{otherwise}\end{cases},
\end{align}
where $\boldsymbol{m}^{\ell}[t]$ and $\boldsymbol{u}^{\ell}[t]$ represent the membrane potentials before firing and after reset, respectively. $\lambda$ denotes a decay factor that controls the decay rate of the membrane potential. $\boldsymbol{s}^{\ell-1}[t]$ represents afferent spikes. $\boldsymbol{W}^{\ell}$ is a weight matrix connecting adjacent layers. 
The Heaviside function~$\Theta(x)$ determines whether a spike is generated based on the relationship between the membrane potential $\boldsymbol{m}^{\ell}[t]$ and the threshold $V_{\mathrm{th}}$.

Upon reaching the output layer, a loss $\mathcal{L}$ is calculated by measuring the discrepancy between the predictions and the targets. 
To update network parameters, gradients of the loss with respect to each parameter are computed through a backward pass. 
This pass starts from the last layer and the last time step and propagates in reverse across all network layers and time steps. 
The gradient of the loss with respect to the weights of layer~$\ell$ is formulated as~\cite{wu2018spatio}:
\begin{align}
\frac{\partial \mathcal{L}}{\partial \boldsymbol{W}^{\ell}}&= \sum_{t=1}^{T}\frac{\partial \mathcal{L}}{\partial \boldsymbol{m}^{\ell}[t]} \frac{\partial \boldsymbol{m}^{\ell}[t]}{\partial \boldsymbol{W}^{\ell}} = \sum_{t=1}^{T}{\boldsymbol{\delta}^{\ell}[t]} {\boldsymbol{s}^{\ell-1}[t]}^{\top}, 
\label{eq:weight}
\end{align}
\begin{equation}
\boldsymbol{\delta}^{\ell}[t]\!=\!\begin{cases} 
\frac{\partial \mathcal{L}}{\partial \boldsymbol{m}^L[T]},  &\ell\!=\!L~\text{and}~t\!=\!T\\ 
\boldsymbol{\delta}^{L}[t+1]\frac{\partial \boldsymbol{m}^{L}[t+1]}{\partial \boldsymbol{m}^{L}[t]}+\frac{\partial \mathcal{L}}{\partial \boldsymbol{m}^L[t]},  &\ell\!=\!L~\text{and}~t\!<\!T\\ 
\boldsymbol{\delta}^{\ell+1}[T]\frac{\partial \boldsymbol{m}^{\ell+1}[T]}{\partial \boldsymbol{s}^{\ell}[T]}\frac{\partial \boldsymbol{s}^{\ell}[T]}{\partial \boldsymbol{m}^{\ell}[T]},  &\ell\!<\!L~\text{and}~t\!=\!T\\ 
\boldsymbol{\delta}^{\ell}[t\!+\!1]\frac{\partial \boldsymbol{m}^{\ell}[t\!+\!1]}{\partial \boldsymbol{m}^{\ell}[t]}\!+\! \boldsymbol{\delta}^{\ell\!+\!1}[t]\!\frac{\partial \boldsymbol{m}^{\ell\!+\!1}[t]}{\partial \boldsymbol{m}^{\ell}[t]},\!\!\!&\!\text{otherwise}\! \end{cases}
\label{eq:grad_m}
\end{equation}
where $\boldsymbol{\delta}^{\ell}[t]\triangleq\frac{\partial \mathcal{L}}{\partial \boldsymbol{m}^{\ell}[t]}$, and $T$ and $L$ denotes the total number of time steps and layers, respectively.
The derivative of the Heaviside function, i.e., $\frac{\partial \boldsymbol{s}_{i}^{\ell}[t]}{\partial \boldsymbol{m}_{i}^{\ell}[t]}$, is zero for all values of~$\boldsymbol{m}_{i}^{\ell}[t]$ except for the case $\boldsymbol{m}_{i}^{\ell}[t]=V_{\mathrm{th}}$, where it becomes infinite. Surrogate functions~\cite{neftci2019surrogate,wu2018spatio} are employed to handle this non-differentiability, i.e., $\mathbb{H}\left(\boldsymbol{m}_i^{\ell}[t]\right) \approx \frac{\partial \boldsymbol{s}_i^{\ell}[t]}{\partial \boldsymbol{m}_i^{\ell}[t]}  $, such as the triangle function: 
\begin{equation}
\label{eq:triangle}
\mathbb{H}\left(\boldsymbol{m}_i^{\ell}[t]\right)=\frac{1}{\gamma^2}\max \left(0,~\gamma-|\boldsymbol{m}_{i}^{\ell}[t]-V_{\mathrm{th}}|\right),
\end{equation}
where $\gamma$ is a hyperparameter typically set to $1$.

According to Eq.~(\ref{eq:grad_m}), BPTT needs to access the gradients from subsequent layers and time steps to compute the gradients for the current layer at the current time step. 
Due to such dependencies across spatial and temporal dimensions, gradient computation in BPTT can only start after the full forward pass is completed. 
Consequently, all the intermediate neuronal states necessary to compute gradients are stored in memory during the forward pass. The resultant memory consumption is proportional to the total number of layers and time steps (see Fig.~\ref{Fig:BPTTvsSTDL} (a) for the illustration). 
Specifically, according to Eq.~(\ref{eq:weight}), afferent spikes $\boldsymbol{s}^{\ell-1}[t]$ are essential for calculating the weight gradients. Therefore, all afferent spikes across both spatial and temporal dimensions, i.e.,  
$\left\{\boldsymbol{s}^{\ell-1}[t]\!\mid\!\ell\!\in\!\{1,\ldots, L\}, t\!\in\!\{1,\ldots,T\}\right\}$, must be stored in memory during the forward pass. Similarly, all membrane potentials, i.e.,  
$\left\{\boldsymbol{m}^{\ell}[t]\!\mid\!\ell\!\in\!\{1,\ldots,L\}, t\!\in\!\{1,\ldots,T\}\right\}$, need to be retained in memory to compute the surrogate gradient~$\mathbb{H}\left(\boldsymbol{m}_i^{\ell}[t]\right)$ according to Eq.~(\ref{eq:triangle}).  
The memory consumption becomes substantial when training large-scale spiking models with a long time duration.

\subsection{Memory-Efficient Training via Decoupling Spatial and Temporal Dependencies}
As aforementioned, the substantial memory requirement of BPTT stems from the spatial and temporal dependencies.
Hence, a natural solution for memory-efficient training is to decouple these dependencies, allowing gradient computations to occur before the entire forward pass is completed.
In this way, cached states can be released from memory, leading to a reduction in the memory requirement. 
This insight aligns with the principles of synaptic plasticity observed by neuroscientists~\cite{hebb1949organization,bi1998synaptic}, where synaptic connections are adjusted locally without waiting for signals from distant and unconnected layers. 
Such local learning rules have been successfully applied in training SNNs. Among them, supervised local learning methods~\cite{kaiser2020synaptic,ma2023ell} have demonstrated superior performance on practical tasks.  As illustrated in Fig.~\ref{Fig:BPTTvsSTDL} (b), these methods enable the online learning of each layer independently using a layer-wise auxiliary classifier, requiring only the states necessary for updating a single layer to be retained in memory at the current time step. 
However, despite their high memory efficiency, they still face a significant accuracy gap compared to BPTT, especially in large-scale SNNs~\cite{ma2023ell}.

\subsection{The Weak Coupling Issue in Local Learning}
\label{subsec:understanding}
Recent studies on online learning have demonstrated that decoupling temporal dependencies, i.e., training each time step independently, has a negligible impact on model accuracy~\cite{meng2023towards,xiao2022online}.
This is because the gradients associated with temporal dependencies contribute minimally to the overall parameter gradients~\cite{meng2023towards}. Motivated by this insight, we hypothesize that the primary cause of the performance gap between supervised local learning and BPTT lies in the decoupling of spatial dependencies. We are thus motivated to address the accuracy loss resulting from the spatial decoupling.

\begin{table}[t]
\centering
\caption{Performance of supervised local learning for training ResNet-18 on CIFAR-10 ($T=1$), which is evenly divided into $K$ independently trained subnetworks. ``$K=1$'' is BPTT training.}
\linespread{1.02}
\setlength{\tabcolsep}{2.5mm}
\begin{tabular}{cccccc}
\hline
          & $K$=1   & $K$=3   & $K$=5   & $K$=7   & $K$=9   \\ \hline
Acc. (\%) & 94.01 & 91.79 & 90.71 & 85.88 & 85.53 \\
Mem. (GB) & 4.56  & 3.36  & 2.63  & 2.29  & 2.29  \\ \hline
\end{tabular}%
\label{tab:varying_k}
\end{table}
\linespread{1.0}
\begin{figure}[t]
\vspace{-2mm}
\centering\includegraphics[width=0.85\linewidth]{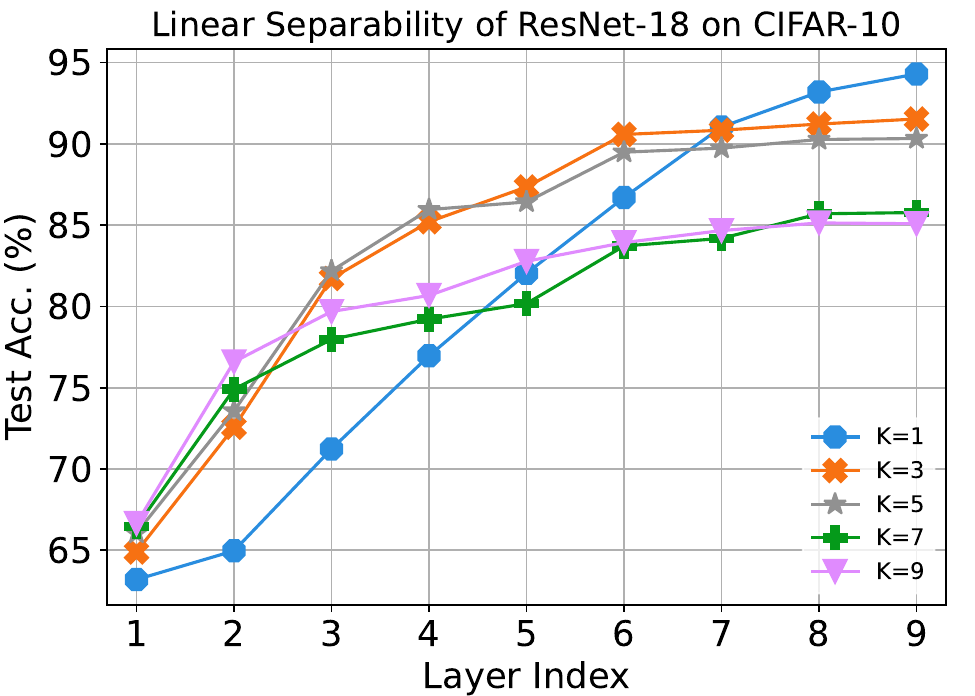}
\vspace{-1.5mm}
\caption{Comparison of layer-wise linear separability among supervised local learning variants that evenly divide ResNet-18 into $K$ subnetworks. } 
\label{fig:varying_k}
\end{figure}

A straightforward solution to mitigate this issue is to reduce the degree of independence across the spatial dimension by decreasing the number of locally trained layers.
However, a naive network partitioning strategy, where the network is evenly divided into several subnetworks, each comprising an equal number of consecutive layers, fails to match the accuracy of BPTT and offers only marginal memory savings.
This insight is supported by empirical results. Specifically, we evaluate supervised local learning on CIFAR-10 using ResNet-18, which is evenly partitioned into $K$ subnetworks, each trained independently with a trainable linear classifier~\cite{ma2023ell}.
As shown in Table~\ref{tab:varying_k}, while reducing the number of subnetworks narrows the accuracy gap to BPTT, the gap remains substantial. Moreover, this improvement comes at the cost of significantly reduced memory efficiency. These findings highlight that achieving a comparable accuracy to BPTT while retaining the memory efficiency of local learning is non-trivial.

To investigate the underlying reasons for the performance degradation, we analyze the linear separability of each layer's representation learned using these local learning variants and compare it with that obtained via BPTT. As shown in Fig.~\ref{fig:varying_k}, local learning variants exhibit distinct patterns of linear separability. Specifically, they tend to learn more linearly separable representations in early layers, but show reduced separability in deeper layers, compared to BPTT. This discrepancy becomes increasingly pronounced as the number of independently trained subnetworks increases. 
This observation indicates that early layers, optimized solely for their local objectives, fail to produce representations that are well aligned with the needs of subsequent layers and benefit the global objective. 
We refer to this as \textit{the weak coupling issue}, where independently trained subnetworks fail to collaborate effectively toward the global objective, which is the major reason for the suboptimal performance of local learning methods.

\section{Spatio-Temporal Decoupled Learning (STDL)}
\label{sec:STDL}
In this section, we propose STDL, a novel memory-efficient training method for SNNs. We first present an overview of STDL, and then describe in detail how it mitigates the weak coupling issue through a greedy network partitioning method and an information-theoretically guided auxiliary network construction approach. Lastly, we explain how STDL incorporates temporal decoupling to further reduce memory overhead without compromising performance.

\subsection{Overview of STDL}
\label{sec:stdl}
As illustrated in Fig.~\ref{Fig:BPTTvsSTDL}(c), STDL enables memory-efficient training of SNNs by decoupling both spatial and temporal dependencies. To decouple spatial dependencies, STDL partitions the full network into several smaller subnetworks, each of which is trained independently with the supervision of its associated auxiliary network. To address the weak coupling issue, STDL introduces a provably optimal greedy network partitioning algorithm that minimizes the number of subnetworks under a given memory constraint. Additionally, STDL incorporates an information-theoretically grounded auxiliary network construction method that guides each subnetwork to produce representations that closely resemble those learned through BPTT, thereby enhancing the coupling across subnetworks.
STDL further decouples temporal dependencies by discarding uninformative temporally dependent gradients, allowing parameter gradients to be computed at each time step. 

Formally, consider an $L$-layer SNN denoted as $\mathcal{F}: \mathcal{X} \rightarrow \mathcal{Y}$.
A mini-batch of input samples $\boldsymbol{x}\in \mathcal{X}$ is processed sequentially through the network layers as: 
\begin{equation}
\mathcal{F}(\boldsymbol{x}) = f^{L} \circ \dots \circ f^{\ell} \circ \dots \circ f^1(\boldsymbol{x}),
\end{equation} where $f^{\ell}(\cdot)$ signifies the operation of the $\ell$-th layer, and $\circ$ is the function composition operator. By the greedy network partitioning algorithm detailed in Section~\ref{sec:greedy_net_partition},
STDL partitions the full network $\mathcal{F}$ into~$K$ subnetworks~$(K \leq L)$ as: 
\begin{equation}
\mathcal{F}(\boldsymbol{x}) = \mathcal{F}^{K} \circ \dots \circ \mathcal{F}^{k} \circ  \dots \circ \mathcal{F}^1(\boldsymbol{x}).
\end{equation}
The propagation of gradients between adjacent subnetworks is prohibited by using the stop-gradient operator $\mathrm{sg}(\cdot)$ to decouple their spatial dependency during training:
\begin{equation}
\boldsymbol{s}^{k}=\mathcal{F}^{k}\left(\mathrm{sg}(\boldsymbol{s}^{k-1})\right).
\end{equation}
Each subnetwork $\mathcal{F}^{k}$ is paired with an auxiliary network $\mathcal{G}^{k}: \mathcal{H} \rightarrow \mathcal{Y}$, which maps its spike-based representation $\boldsymbol{s}^{k} \in \mathcal{H}$ to a prediction. The construction of $\mathcal{G}^{k}$ is described in Section~\ref{sec:aux_net_construct}. The pair $(\mathcal{F}^{k}, \mathcal{G}^{k})$ is jointly optimized using a local loss function $\mathcal{L}^{k}(\mathcal{G}^{k}(\boldsymbol{s}^{k}), \boldsymbol{y})$ that compares the local prediction with the target $\boldsymbol{y} \in \mathcal{Y}$. This loss is computed as the sum of instantaneous losses over the time window:
\begin{equation}
\mathcal{L}^k\left(\mathcal{G}^{k}(\boldsymbol{s}^{k}), \boldsymbol{y}\right)=\sum_{t=1}^T\mathcal{L}^k[t]\left(\mathcal{G}^{k}(\boldsymbol{s}^{k}[t]), \boldsymbol{y}[t]\right).    
\end{equation}
To decouple temporal dependencies, STDL enables the online learning of each subnetwork and its auxiliary network, as detailed in Section~\ref{sec:tem_dec}. Note that the last subnetwork~$\mathcal{F}^{K}$ is not paired with an auxiliary network, as it already includes the output classifier. All auxiliary networks are used only during training and discarded during inference.

\subsection{Greedy Network Partitioning for a Minimal Number of Subnetworks}
\label{sec:greedy_net_partition}

As shown in Section~\ref{subsec:understanding}, increasing the number of subnetworks exacerbates the weak coupling issue and leads to noticeable accuracy degradation. Additionally, due to different memory requirements across layers in SNN architectures~\cite{simonyan2014very,he2016deep,fang2021deep}, it is often possible to group several successive layers into a single subnetwork without affecting the peak memory footprint. These insights motivate us to minimize the number of subnetworks under a given memory budget.

To formalize this objective, let $\boldsymbol{P}\!=\!\{p_1, p_2, \ldots, p_K\}$ denote the set of subnetwork boundary indices, where each $p_k$ corresponds to the index of the last layer in the $k$-th subnetwork, and $K$ is the total number of subnetworks. Let $\mathrm{mem}_\ell$ represent the memory footprint of the $\ell$-th layer, and let $M_{\mathrm{aux}}$ be the maximum allowable memory footprint per subnetwork. The objective is to partition the $L$-layer network into the smallest number of subnetworks such that the memory footprint of each subnetwork does not exceed $M_{\mathrm{aux}}$. This problem can be formally defined as:
\begin{align}
\min_{\boldsymbol{P}}& \quad K \nonumber\\
\textrm{subject to} \sum_{\ell=p_{k-1}+1}^{p_k}\mathrm{mem}_{\ell} & \leq M_{\mathrm{aux}}, \quad k = 1, \ldots, K, \nonumber\\ 
1\leq p_1<p_2&<\dots<p_K=L.
\label{eq:net_partition_objective}
\end{align}
Solving this problem via brute force is computationally intractable, as it requires evaluating all possible partitions of the $L$ layers, resulting in exponential complexity. To address this, we propose a greedy network partitioning algorithm that efficiently constructs the optimal solution in a single pass. 

Specifically, the greedy algorithm initializes at the layer following the previous subnetwork, and incrementally adds layers into the current subnetwork until the total memory footprint of the included layers surpasses the given memory constraint. Any remaining layers are then assigned to the next subnetwork with the same iterative process. This partitioning process continues until all layers are allocated into subnetworks. The pseudocode for the greedy algorithm is provided in Algorithm~\ref{alg:greedy_partition}. The following proposition establishes the optimality of this greedy algorithm.
\begin{prop}
The greedy network partitioning algorithm described in Algorithm~\ref{alg:greedy_partition} yields an optimal solution to the optimization problem defined in Eq.~\ref{eq:net_partition_objective}.
\label{prop:greedy_net_partition}
\end{prop} 
\begin{proof}
The feasibility of the solution generated by the greedy algorithm is evident, as it consistently ensures that the memory footprint of each subnetwork does not exceed the constraint and that the subnetwork boundaries also meet the constraint.

Next, we prove the optimality of the solution. Assume that $\boldsymbol{P}\!=\!\{p_1, p_2, \ldots, p_K\}$ produced by the greedy algorithm is not optimal, and let $\boldsymbol{O}\!=\!\{o_1,o_2, \ldots, o_M\}$ be an optimal solution. Let $r$ be the largest index where the partitions $\boldsymbol{P}$ and $\boldsymbol{O}$ coincide, such that $p_1\!=\!o_1, \ldots, p_r\!=\!o_r$. The difference arises at $p_{r+1}$ and $o_{r+1}$, where $p_{r+1}\!>\!o_{r+1}$, because the greedy algorithm includes more layers. Since the greedy algorithm's $r\!+\!1$-th subnetwork, which includes layers up to $p_{r+1}$, does not violate the memory constraint, the corresponding range in $\boldsymbol{O}$ must adhere to the same constraint as they start with the same layer. As a result, we can adjust $o_{r+1}$ to equal $p_{r+1}$ without increasing the number of subnetworks or exceeding the memory constraint. Applying this adjustment iteratively to transform $\boldsymbol{O}$ into $\boldsymbol{P}$ does not degrade the quality of the solution. Therefore, the solution yielded by the greedy algorithm is optimal.
\end{proof}

\begin{algorithm}[!t]
\caption{Greedy Network Partitioning}
\label{alg:greedy_partition}
\KwIn{Memory footprint of each layer~$\{\mathrm{mem}_\ell\}_{\ell=1}^L$, total number of layers $L$, memory constraint per subnetwork $M_{\mathrm{sub}}$}
\KwOut{\!Subnetwork boundaries $\!\boldsymbol{P}\!=\!\{p_1, p_2, \!\ldots,\!p_K\}$}

$\boldsymbol{P} \gets \emptyset$\;
$current\_mem \gets 0$\;

\For{$\ell = 1$ \KwTo $L$}{
    $current\_mem \gets current\_mem + \mathrm{mem}_\ell$\;
    \If{$current\_mem > M_{\mathrm{sub}}$}{
        $p \gets \ell - 1$\;
        $\boldsymbol{P} \gets \boldsymbol{P} \cup \{p\}$\;
        $current\_mem \gets \mathrm{mem}_\ell$\;
    }
}
$\boldsymbol{P} \gets \boldsymbol{P} \cup \{L\}$\;
\Return{$\boldsymbol{P}$}
\end{algorithm}

\subsection{Auxiliary Network Construction for Representation Alignment} 
\label{sec:aux_net_construct}
The design of auxiliary networks is crucial for partitioned subnetworks to contribute collaboratively to the global objective. 
To promote the coupling among subnetworks, we aim to construct auxiliary networks that enable the representations of their subnetworks to align with those produced by BPTT. This is reasonable as the subnetworks trained using BPTT are naturally coupled. To achieve this, we introduce a structure prior into the auxiliary networks.
Specifically, for a subnetwork, we select a subset of its subsequent layers as its auxiliary network.
In this way, the resultant auxiliary network can incorporate critical transformations in the subsequent layers, thereby enabling the subnetwork's representation to be similar to that of BPTT. However, under a given memory constraint, it remains non-trivial to determine which subsequent layers should be selected to construct the auxiliary network. To address this challenge, we analyze the discrepancy between representations produced by the auxiliary network and BPTT from an information-theoretic perspective~\cite{shwartz2017opening}. Building on this analysis, we propose to select the subsequent layers that maximize the auxiliary network's capacity, i.e., depth and width, within the memory budget, as explained below.

Without loss of generality, we consider the $k$-th subnetwork. Let $\boldsymbol{s}_{\mathrm{local}}^k$ denote the spike-based representation learned by its auxiliary network $\mathcal{G}^k$, and let $\hat{\boldsymbol{y}}_{\mathrm{local}}^k$ denote the local prediction generated by $\mathcal{G}^k$. Similarly, we define $\boldsymbol{s}_{\mathrm{global}}^k$ as the representation learned via BPTT and $\hat{\boldsymbol{y}}_{\mathrm{global}}$ as the global prediction. Then, we quantify the difference between $\boldsymbol{s}_{\mathrm{local}}^k$ and $\boldsymbol{s}_{\mathrm{global}}^k$ from an information-theoretic perspective. Specifically, we regard $\boldsymbol{s}_{\mathrm{local}}^k$, $\boldsymbol{s}_{\mathrm{global}}^k$, $\hat{\boldsymbol{y}}_{\mathrm{local}}^k$, $\hat{\boldsymbol{y}}_{\mathrm{global}}$, along with the input~$\boldsymbol{x}$ and the label $\boldsymbol{y}$, as random variables. The information encoded by $\boldsymbol{s}_{\mathrm{local}}^k$ and $\boldsymbol{s}_{\mathrm{global}}^k$ is then characterized using two mutual information quantities~\cite{shwartz2017opening}: $I(\boldsymbol{s}^k; \boldsymbol{x})$, quantifying the information related to the input $\boldsymbol{x}$, and $I(\boldsymbol{s}^k; \boldsymbol{y})$, quantifying the information related to the label $\boldsymbol{y}$. In the following analysis, we focus exclusively on the discrepancy between $I(\boldsymbol{s}_{\mathrm{local}}^k; \boldsymbol{y})$ and $I(\boldsymbol{s}_{\mathrm{global}}^k; \boldsymbol{y})$. The discrepancy between $I(\boldsymbol{s}_{\mathrm{local}}^k; \boldsymbol{x})$ and $I(\boldsymbol{s}_{\mathrm{global}}^k; \boldsymbol{x})$ can be derived using an analogous procedure.

\begin{figure}[!tb]
\centering\includegraphics[width=0.85\linewidth]{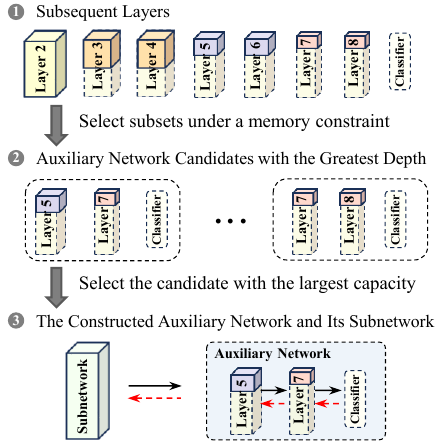}
\caption{\textbf{Illustration of auxiliary network construction in STDL.} For each subnetwork, the auxiliary network is constructed by selecting a subset of its subsequent layers. The subnet has the largest expressive capacity, i.e., depth and width, while not exceeding a given memory constraint. The shaded area within each layer represents its memory footprint during training.}
\label{fig:sptial_dec}
\end{figure}

Our objective is to find an auxiliary network $\mathcal{G}^k$ that can guide its subnetwork to generate $\boldsymbol{s}_{\mathrm{local}}^k$ closely resembling~$\boldsymbol{s}_{\mathrm{global}}^k$ under the memory constraint $M_{\mathrm{sub}}$. This can be formulated by minimizing the difference between the two mutual information quantities:
\begin{align} 
\min_{\mathcal{G}^k} \quad | I(\boldsymbol{s}_{\mathrm{global}}^k; \boldsymbol{y}) &- I(\boldsymbol{s}_{\mathrm{local}}^k; \boldsymbol{y}) | \nonumber \\
\textrm{subject to} \quad M_{\mathcal{G}^k} & \leq M_{\mathrm{sub}},
\label{eq:min_mi}
\end{align}
where $M_{\mathcal{G}^k}$ denotes the memory footprint of the auxiliary network $\mathcal{G}^k$. Note that $I(\boldsymbol{s}_{\mathrm{local}}^k; \boldsymbol{y}) \leq I(\boldsymbol{s}_{\mathrm{global}}^k; \boldsymbol{y})$, where the equality holds if and only if the auxiliary network is exactly equivalent to the full subsequent layers. However, due to the memory constraint, auxiliary networks typically possess lower capacity, and thus inherently retain less information about~$\boldsymbol{y}$. Since the subsequent layers of the $k$-th subnetwork are fixed, $I(\boldsymbol{s}_{\mathrm{global}}^k; \boldsymbol{y})$ is constant. Consequently, minimizing Eq.~\ref{eq:min_mi} is equivalent to maximizing $I(\boldsymbol{s}_{\mathrm{local}}^k; \boldsymbol{y})$. Instead of directly analyzing the maximization, we seek to maximize a surrogate objective, specifically a lower bound on the mutual information~$I(\boldsymbol{s}_{\mathrm{local}}^k; \boldsymbol{y})$. We achieve this by leveraging the data processing inequality (DPI)~\cite{shwartz2017opening}. Specifically, with the Markov chain $\boldsymbol{y} \rightarrow \boldsymbol{x} \rightarrow \boldsymbol{s}^k \rightarrow \hat{\boldsymbol{y}}$, the DPI guarantees:
\begin{equation} I(\boldsymbol{s}_{\mathrm{local}}^k; \boldsymbol{y}) \geq I(\hat{\boldsymbol{y}}_{\mathrm{local}}^k; \boldsymbol{y}). \end{equation}
This suggests that the prediction $\hat{\boldsymbol{y}}$, being a function of the representation $\boldsymbol{s}^k$, cannot contain more information about $\boldsymbol{y}$ than the representation itself. As a result, minimizing the mutual information gap in Eq.~\ref{eq:min_mi} is equivalent to maximizing the lower bound~$I(\hat{\boldsymbol{y}}_{\mathrm{local}}^k; \boldsymbol{y})$. The following proposition establishes how this quantity can be increased by expanding the expressive capacity of the auxiliary network.
\begin{prop}
\label{prop:mutual_monotone}
Let $\mathcal{G}_{\boldsymbol{\theta}}^k$ be the $k$-th auxiliary network, parameterized by $\boldsymbol{\theta} \in \mathbb{R}^{D \times W}$ with depth $D$ and width $W$, and assume it is sufficiently trained without significant overfitting or underfitting. Let $\hat{\boldsymbol{y}}_{\mathrm{local}}^k = \mathcal{G}_{\boldsymbol{\theta}}^k(\boldsymbol{s}^k)$ denote its prediction given input $\boldsymbol{s}^k$. Consider an expanded auxiliary network $\mathcal{G}_{\boldsymbol{\theta}'}^k$ with parameters $\boldsymbol{\theta}' \in \mathbb{R}^{D' \times W'}$, satisfying $D' \geq D$ and $W' \geq W$. Assume $\mathcal{G}_{\boldsymbol{\theta}'}^k$ is also sufficiently trained, and let $\hat{\boldsymbol{y}}_{\mathrm{local}}^{k'} = \mathcal{G}_{\boldsymbol{\theta}'}^k(\boldsymbol{s}^k)$ denote its prediction. Then, the mutual information between the predictions and the label satisfies:
\begin{equation}
I\left(\hat{\boldsymbol{y}}_{\mathrm{local}}^k; \boldsymbol{y} \right) \leq I\left(\hat{\boldsymbol{y}}_{\mathrm{local}}^{k'}; \boldsymbol{y} \right).
\end{equation}
\end{prop}

\begin{proof}
We first show that there exists a parameter configuration $\boldsymbol{\theta}^\prime$ such that the expanded auxiliary network $\mathcal{G}_{\boldsymbol{\theta}^\prime}^k$ computes the same function as the original network $\mathcal{G}_{\boldsymbol{\theta}}^k$ for any input $\boldsymbol{s}^k$, i.e.,~$\mathcal{G}_{\boldsymbol{\theta}^\prime}^k(\boldsymbol{s}^k) = \mathcal{G}_{\boldsymbol{\theta}}^k(\boldsymbol{s}^k),  \forall \boldsymbol{s}^k.$

We construct the $\boldsymbol{\theta}'$ in three steps. First, copy the original parameters into the corresponding sub-block of the expanded parameter matrix, i.e., $\boldsymbol{\theta}^\prime_{:D,\,:W} = \boldsymbol{\theta}$. Second, set the additional parameters introduced by the increased width to zero, i.e., $\boldsymbol{\theta}^\prime_{:D,\,W:} = 0$. Third, configure the additional layers to act as identity mappings. This can be achieved by using residual connections with setting their weights to zero, i.e., $\boldsymbol{\theta}^\prime_{D:, :} = 0$.

This construction guarantees that the expanded network exactly replicates the function computed by the original network. Then, consider training the expanded network $\mathcal{G}_{\boldsymbol{\theta}'}^k$ from this initialization. Due to its increased depth and width, the expanded network has the potential to discover a more informative representation with respect to the label. Consequently, the mutual information between its output and the label can only improve or remain the same. The proof is completed.
\end{proof}
Proposition~\ref{prop:mutual_monotone} establishes that the mutual information~$I(\hat{\boldsymbol{y}}_{\mathrm{local}}^k; \boldsymbol{y})$ is monotonically non-decreasing with respect to the depth and width of the auxiliary network. Motivated by this insight, we design each auxiliary network to maximize its expressive capacity, i.e., its depth and width,  within the memory budget $M_{\mathrm{sub}}$. Since increasing network depth yields exponentially greater expressive capacity than increasing width~\cite{eldan2016power,lecun2015deep}, we prioritize depth expansion over width. An illustration of our method is shown in Fig.~\ref{fig:sptial_dec}.

Concretely, we first determine the maximum number of layers from the subsequent layers without violating the memory constraint. We then construct the auxiliary network by selecting layers accordingly to match this maximum allowable depth. Among candidate configurations with the same depth, we prioritize the one with wider layers, as it provides additional expressive capacity under the same depth.

\subsection{Decoupling Temporal Dependencies by Omitting Temporally Dependent Gradients}
\label{sec:tem_dec} 
While auxiliary networks allow each subnetwork to be trained independently, the computation of parameter gradients within each subnetwork still depends on future time steps due to temporal dependencies. As a result, gradients cannot be computed immediately at the current time step, and the memory footprint remains proportional to the total number of time steps. To address this, STDL explicitly discards temporally dependent gradients that contribute negligibly to the overall parameter update. By ignoring these low-impact temporal gradients, STDL enables the online learning of each subnetwork and its auxiliary network, thereby further reducing the memory footprint without sacrificing accuracy. In the following, we detail how STDL computes gradients for a given layer $j$ in the $k$-th subnetwork and its auxiliary network. 

To explicitly illustrate the presence of temporally dependent gradients in BPTT, we expand the gradient expression in Eq.~(\ref{eq:grad_m}) by unfolding the term $\frac{\partial \mathcal{L}^k}{\partial \boldsymbol{m}^{\ell}[t+1]}\frac{\partial \boldsymbol{m}^{\ell}[t+1]}{\partial \boldsymbol{m}^{\ell}[t]}$ along the temporal dimension of the next layer. This yields the following gradient expression for the membrane potential $\boldsymbol{m}^j[t]$ in layer~$j$ of subnetwork $k$:
\begin{align}
\frac{\partial \mathcal{L}^k}{\partial \boldsymbol{m}{^{j}}[t]} &= \frac{\partial \mathcal{L}^k}{\partial \boldsymbol{m}{^{j+1}}[t]}\frac{\partial \boldsymbol{m}^{j+1}[t]}{\partial \boldsymbol{m}^{j}[t]} \nonumber\\
&+ \underbrace{\sum_{t^\prime=t+1}^T
\frac{\partial \mathcal{L}^k}{\partial \boldsymbol{m}{^{j+1}}[t^\prime]} \frac{\partial \boldsymbol{m}{^{j+1}}[t^\prime]}{\partial \boldsymbol{m}{^{j}}[t^\prime]}\frac{\partial \boldsymbol{m}{^{j}}[t^\prime]}{\partial \boldsymbol{m}{^{j}}[t]}}_{\text{Temporally Dependent Gradients}}.
\label{eq:td_grad_m}
\end{align}
Eq.~(\ref{eq:td_grad_m})~reveals that the gradient at the current time step~$t$ is influenced by a summation over all future time steps $t^\prime > t$. These temporally dependent gradients arise primarily from the term $\frac{\partial \boldsymbol{m}^j[t^\prime]}{\partial \boldsymbol{m}^j[t]}$, which reflects the backward temporal influence from future states $t^\prime$ back to the current time $t$.
When omitting the reset process, as is common in prior studies~\cite{fang2021incorporating,fang2021deep}, this term can be approximated as $\frac{\partial \boldsymbol{m}^j[t^\prime]}{\partial \boldsymbol{m}^j[t]} \approx \lambda^{t^\prime - t}$, where $\lambda$ is the decay factor of the membrane potential. As the temporal gap~$t^\prime - t$ increases, this term decays exponentially, diminishing its contribution to the overall gradient.

To enable gradient computation at the current time step $t$, we discard temporally dependent gradients, as their contribution to the overall gradient is negligible. Formally,  we define $\boldsymbol{\psi}^j[t]$ as the total influence of $\boldsymbol{m}{^{j}}[t]$ on the local instantaneous loss~$\mathcal{L}^k[t]$, and it can be calculated recursively as:
\begin{equation}
\boldsymbol{\psi}^j[t]\!=\!\begin{cases} 
\frac{\partial \mathcal{L}^k[t]}{\partial \boldsymbol{m}^J[t]},  &j\!=\!J\\ 
\boldsymbol{\psi}^{j+1}[t]\frac{\partial \boldsymbol{m}^{j+1}[t]}{\partial \boldsymbol{m}^{j}[t]},  &j\!<\!J\\ 
\end{cases}
\end{equation}
where $J$ represents the last layer of the auxiliary network. Using this quantity, the gradient of the local loss with respect to the weights of layer $j$ is computed as:
\begin{equation}
\frac{\partial \mathcal{L}^k}{\partial \boldsymbol{W}^{j}}= \sum_{t=1}^T{\boldsymbol{\psi}^j[t]}\frac{\partial \boldsymbol{m}^{j}[t]}{\partial\boldsymbol{W}^{j}} = \sum_{t=1}^T{\boldsymbol{\psi}^j[t]} {\boldsymbol{s}^{j-1}[t]}^{\top}.
\label{eq:tlg}
\end{equation}
Both $\boldsymbol{\psi}^j[t]$ and $\boldsymbol{s}^{j-1}[t]$ in Eq.~(\ref{eq:tlg}) are readily available at time step $t$, thereby allowing the online learning of each subnetwork and its auxiliary network. This effectively decouples temporal dependencies and ensures that the memory footprint is independent of the total number of time steps.

\begin{table*}[t]
\begin{threeparttable}
\centering
\caption{Comparison of STDL and other learning methods on four image classification datasets. Averaged accuracy and standard deviation are reported from three independent trials. The GPU memory footprint is measured with a batch size of 512.}
\linespread{1.02}
\setlength{\tabcolsep}{1.1mm}
\begin{tabular}{ccccccccccc}
\toprule
\multirow{3}{*}{Method} &
  \multicolumn{2}{c}{Decoupled in} &
  \multicolumn{2}{c}{CIFAR-10} &
  \multicolumn{2}{c}{CIFAR-100} &
  \multicolumn{2}{c}{SVHN} &
  \multicolumn{2}{c}{ImageNet} \\ \cmidrule(l){4-11} 
 &
  \multirow{2}{*}{\begin{tabular}[c]{@{}c@{}}Space\end{tabular}} &
  \multirow{2}{*}{Time} &
  \multicolumn{2}{c}{ResNet-18 (T=4)} &
  \multicolumn{2}{c}{ResNet-19 (T=4)} &
  \multicolumn{2}{c}{VGG16 (T=4)} &
  \multicolumn{2}{c}{SEWResNet-34 (T=4)} \\ \cmidrule(l){4-11} 
                     &   &   & Acc. (\%)       & Mem. (GB) & Acc. (\%)       & Mem. (GB) & Acc. (\%)       & Mem. (GB) & Acc. (\%)       & Mem. (GB) \\ \midrule
BPTT~\cite{wu2018spatio}                 & \XSolidBrush & \XSolidBrush & 95.12+0.15 & 15.45    & 78.83$\pm$0.10 & 39.44    & 96.09$\pm$0.14 & 7.84    & 70.12 & 424.35   \\
SLTT~\cite{meng2023towards}                 & \XSolidBrush & \Checkmark & 95.07$\pm$0.04 & 4.88~$(\downarrow3.2\times)$     & 78.89$\pm$0.16          & 12.38~$(\downarrow3.2\times)$    & 96.61$\pm$0.06          & 2.54~$(\downarrow3.1\times)$        & 70.11          & 144.37~$(\downarrow2.9\times)$        \\
DECOLLE~\cite{kaiser2020synaptic}              & \Checkmark & \Checkmark & 62.05$\pm$0.63          & 2.76~$(\downarrow5.6\times)$         & 33.89$\pm$0.39          & 6.07~$(\downarrow6.5\times)$        & 79.04$\pm$0.31          & 1.81~$(\downarrow4.3\times)$        & 16.26          & 86.33~$(\downarrow4.9\times)$        \\
ELL~\cite{ma2023ell}                  & \Checkmark & \Checkmark & 88.40$\pm$0.13          & 2.76~$(\downarrow5.6\times)$        & 66.48$\pm$0.24          & 6.07~$(\downarrow6.5\times)$        & 94.79$\pm$0.03         & 1.81~$(\downarrow4.3\times)$        & 48.99          & 87.83~$(\downarrow4.8\times)$        \\
\baseline{}STDL & \baseline{}\Checkmark & \baseline{}\Checkmark & \baseline{}94.99$\pm$0.13 & \baseline{}3.51~$(\downarrow4.4\times)$     & \baseline{}78.91$\pm$0.11 & \baseline{}8.31~$(\downarrow4.7\times)$     & \baseline{}96.71$\pm$0.07          & \baseline{}1.82~$(\downarrow4.3\times)$        & \baseline{}69.87        & \baseline{}106.53~$(\downarrow4.0\times)$        \\ \bottomrule
\end{tabular}
\label{tab:static_img}
\end{threeparttable}
\end{table*}
\linespread{1.0}

\begin{table}[!t]
	\centering
	\caption{Comparison of STDL and BPTT-based methods with the accuracy results sourced from their original papers.}
\linespread{1.02}
\setlength{\tabcolsep}{1.5mm}
\begin{tabular}{ccccc}
\toprule
Dataset     & Method               & Network               & T  & Accuracy (\%)            \\ \midrule
CIFAR-10     & LTL-Online~\cite{2022ltl}           & ResNet-20             & 16          & 93.15               \\
            & OTTT\cite{xiao2022online}                 & VGG-11                & 6           & 93.52$\pm$0.06          \\
            & Dspike~\cite{li2021differentiable}               & ResNet-20             & 6           & 94.25$\pm$0.07          \\
            & GLIF~\cite{yao2022glif}                 & ResNet-18             & 4           & 94.67$\pm$0.05          \\
            & \baseline{}\textbf{STDL} & \baseline{}\textbf{ResNet-18}    & \baseline{}\textbf{4}  & \baseline{}\textbf{94.99$\pm$0.13} \\
            \cline{2-5}
            & TET~\cite{deng2021temporal}                  & ResNet-19             & 4           & 94.44$\pm$0.08          \\
            & IM-Loss~\cite{guo2022imloss}              & ResNet-19             & 4           & 95.40$\pm$0.08          \\
            & RecDis~\cite{guo2022recdis}               & ResNet-19             & 4           & 95.53$\pm$0.05          \\
            & \baseline{}\textbf{STDL} & \baseline{}\textbf{ResNet-19}    & \baseline{}\textbf{4}  & \baseline{}\textbf{96.15$\pm$0.07} \\\hline
CIFAR-100   & Dspike~\cite{li2021differentiable}               & ResNet-18             & 4           & 73.35$\pm$0.14          \\
            & RecDis~\cite{guo2022recdis}               & ResNet-18             & 4           & 74.10$\pm$0.13          \\
            & TET~\cite{deng2021temporal}                  & ResNet-18             & 6           & 74.72$\pm$0.28          \\
            & \baseline{}\textbf{STDL} & \baseline{}\textbf{ResNet-18}    & \baseline{}\textbf{4}  & \baseline{}\textbf{75.35$\pm$0.22} \\ \cline{2-5}
            & STBP-tdBN~\cite{zheng2021going}            & ResNet-19             & 4           & 70.86$\pm$0.22          \\
            & RecDis~\cite{guo2022recdis}               & ResNet-19             & 4           & 74.10$\pm$0.13          \\
            & TET~\cite{deng2021temporal}                  & ResNet-19             & 4           & 74.47$\pm$0.15          \\
            & \baseline{}\textbf{STDL} & \baseline{}\textbf{ResNet-19}    & \baseline{}\textbf{4}  & \baseline{}\textbf{78.91$\pm$0.11} \\\hline
ImageNet    & OTTT~\cite{xiao2022online}                 & NFResNet-34           & 6           & 65.15             \\
            & PLIF~\cite{fang2021deep}                 & SEWResNet-34          & 4           & 67.04             \\
            & RecDis~\cite{guo2022recdis}                 & ResNet-34           & 6           & 67.33             \\ 
            & IM-Loss~\cite{guo2022imloss} & ResNet-34 & 6 & 67.43 \\
            & GLIF~\cite{yao2022glif}                 & ResNet-34           & 4           & 67.52             \\
            & TET~\cite{deng2021temporal}                  & SEWResNet-34            & 4           & 68.00            \\
            &TEBN~\cite{duan2022temporal} & SEWResNet-34 & 4 & 68.28 \\
            & \baseline{}\textbf{STDL} & \baseline{}\textbf{SEWResNet-34} & \baseline{}\textbf{4}  & \baseline{}\textbf{69.87}          \\ \bottomrule
\end{tabular}
\label{tab:res_sota_cmp}
 \end{table}
\linespread{1.0}
\section{Experiments}
\label{sec:exp_res}

In this section, we evaluate the effectiveness of STDL on image classification (Section~\ref{sec:static_image_exp}) and event-based vision recognition (Section~\ref{sec:event_exp}). Then, we provide ablation studies and analytical results to provide a better understanding of STDL (Section~\ref{sec:mem_ana}--\ref{sec:memory_acc_trade}). 
Finally, we examine its generalization ability to different spiking neuron models (Section~\ref{sec:diff_neuron_model}).

%
\subsection{Performance Evaluation on Image Classification Datasets}
\label{sec:static_image_exp}
\subsubsection{Experiment Setup}

\textbf{Datasets.} We use four public datasets for the experiments on image classification, including CIFAR-10~\cite{krizhevsky2009learning}, CIFAR-100~\cite{krizhevsky2009learning}, SVHN~\cite{netzer2011reading}, and ImageNet~\cite{deng2009imagenet}.
Following~\cite{deng2021temporal,li2021differentiable,wang2023adaptive}, we apply the autoaugment and cutout data augmentation techniques to CIFAR-10 and CIFAR-100 and employ standard data augmentation techniques on SVHN and ImageNet. More details of the datasets can be found in Appendix~\ref{app:datasets}. 
In line with previous works~\cite{ma2023ell,meng2023towards,fang2021incorporating}, we convert images into temporal sequences by replicating the image across a given number of time steps. These sequences are then directly fed into SNNs, with the first layer serving as an encoding layer that converts pixel values into spikes.

\textbf{Baselines.}
We compare STDL with four learning methods, including BPTT~\cite{wu2018spatio}, SLTT\cite{meng2023towards}, DECOLLE~\cite{kaiser2020synaptic}, and ELL~\cite{ma2023ell}. They have different decoupled dimensions as shown in Table~\ref{tab:static_img}. 
Specifically, DECOLLE and ELL are local learning methods that decouple both spatial and temporal dependencies; SLTT is an online learning method that only decouples temporal dependencies; and BPTT is an end-to-end training method without decoupling.
For fair comparisons, we re-implement these baselines and use consistent training settings. 

\textbf{Training Configurations.}
To validate that the effectiveness of STDL is not tied to a specific network architecture, we adopt different spiking architectures, including ResNet-18~\cite{he2016deep}, ResNet-19~\cite{he2016deep}, VGG16~\cite{simonyan2014very}, and SEWResNet-34~\cite{fang2021deep}.
To apply STDL under a given memory constraint, we first ensure that each auxiliary network contains at least one layer in addition to the classifier layer, and then construct subnetworks within the remaining memory budget, followed by the construction of their corresponding auxiliary networks. Detailed configurations of the resulting subnetworks and auxiliary networks are provided in Appendix~\ref{app:subn_aux_nets}. 
Note that we treat the minimal indivisible unit of a network as a single layer, which corresponds to a residual block and a convolutional layer in ResNets and VGGs, respectively.

For all experiments, we use SGD with momentum and apply a cosine annealing schedule for learning rate adjustment. On ImageNet, we adopt a two-stage training procedure: we first pre-train the SNN using a single time step for 100 epochs, and subsequently fine-tune it with 4 time steps for an additional 10 epochs. Additional training configurations are provided in Appendix~\ref{app:train_setup} and Table~\ref{tab:train_details}.

\subsubsection{STDL Achieves Comparable Accuracy to BPTT with Substantial Memory Savings}
Table~\ref{tab:static_img} summarizes the performance of the proposed STDL method in comparison to the baselines. STDL consistently achieves accuracy comparable to or better than that of BPTT, while significantly saving the GPU memory footprint for training by $4.0\times$-$4.7\times$. 
This superior accuracy and memory advantage hold across various datasets and architectures, highlighting the strong generalizability of STDL. 
Compared to SLTT, STDL achieves approximately 30\% lower GPU memory consumption while maintaining comparable accuracy. This demonstrates that STDL’s spatial decoupling strategy effectively improves memory efficiency without compromising accuracy.

\subsubsection{STDL Significantly Outperforms Local Learning Methods} The accuracy gains of STDL are especially pronounced when compared to existing local learning methods such as DECOLLE and ELL. For example, on CIFAR-100 with ResNet-19, STDL improves accuracy by over 45\% and 12\% compared to DECOLLE and ELL, respectively. These results clearly demonstrate the effectiveness of STDL in addressing the limitations of local learning.

\subsubsection{STDL Effectively Scales to Large-Scale Datasets}
To validate the scalability of STDL with a growing amount of training data, we conduct experiments on the challenging ImageNet dataset using SEWResNet-34~\cite{fang2021deep}. The results are shown in Table~\ref{tab:static_img}. 
It can be observed that local learning methods struggle to perform well on this large-scale dataset.
In contrast, STDL achieves competitive accuracy to BPTT, and outperforms DECOLLE by over 50\% and ELL by over 20\%. 
Moreover, STDL significantly reduces the GPU memory usage during training by $4.0\times$ compared to BPTT while maintaining comparable accuracy. These results clearly demonstrate the scalability of STDL in handling large-scale datasets.

\subsubsection{STDL Achieves Competitive Accuracy Compared to BPTT-Based Methods} 
We compare STDL with existing BPTT-based training methods,  including surrogate gradients~\cite{li2021differentiable}, loss functions~\cite{deng2021temporal,guo2022recdis,guo2022imloss}, neuron models~\cite{fang2021incorporating,yao2022glif}, and normalization methods~\cite{zheng2021going,duan2022temporal}.
As shown in Table~\ref{tab:res_sota_cmp}, STDL consistently achieves competitive accuracy across different architectures and datasets. 
This reaffirms the superior accuracy of STDL.

\subsection{Performance Evaluation on Event-Based Vision Datasets}
\label{sec:event_exp}
\subsubsection{Experiment Setup}
\textbf{Datasets.} We further evaluate the performance of STDL on three event-based vision recognition datasets, including CIFAR10-DVS~\cite{li2017CIFAR10}, GAIT-DAY-DVS~\cite{gait_day_pami}, and HAR-DVS~\cite{2024_hardvs}. 
CIFAR10-DVS is built by scanning images in CIFAR-10 with repeated closed-loop movement using a DVS camera. GAIT-DAY-DVS captures the gait patterns of 20 volunteers under daylight conditions. Both CIFAR10-DVS and GAIT-DAY-DVS have the same spatial resolution of $128\times128$. HAR-DVS is an event-based human activity recognition~(HAR) dataset with a spatial resolution of $346\times260$. It is currently the largest event-based HAR dataset, containing $300$ activity classes and a total of $107,646$ samples. Fig.~\ref{fig:dvs_samples} visualizes randomly selected samples in the three datasets.
We employ the standard preprocessing pipeline provided by the SpikingJelly~\cite{2023spikingjelly} framework to transform events into frames.

\textbf{Training Configurations.} We use VGG-11~\cite{simonyan2014very} on CIFAR10-DVS and ResNet-18~\cite{he2016deep} on GAIT-DAY-DVS and HAR-DVS. We train these networks with the four baselines and STDL under consistent training settings, which are detailed in Appendix~\ref{app:train_setup} and Table~\ref{tab:train_details}. 
For training ResNet-18 with STDL, we adopt the same configuration of subnetworks and auxiliary networks as that in the previous experiments to examine its generalizability.

\begin{table}[t]
\centering
\caption{Comparison of STDL and baselines on three event-based vision datasets. 
The GPU memory is measured with a batch size of $100$.}
\linespread{1.02}
\setlength{\tabcolsep}{0.8mm}
\begin{tabular}{ccccc}
\toprule
Dataset      & Method      & T  & Accuracy (\%)   & GPU Mem. (GB) \\ \hline
CIFAR10-DVS  & Dspike~\cite{li2021differentiable}      & 10 & 75.40$\pm$0.05 & -        \\
             & OTTT~\cite{xiao2022online}        & 10 & 76.27$\pm$0.05 & -        \\
             & TET~\cite{deng2021temporal}         & 10 & 77.33$\pm$0.21 & -        \\ \cline{2-5} 
             & BPTT~\cite{wu2018spatio}        & 10 & 77.87$\pm$0.38 & 99.09    \\
             & SLTT~\cite{meng2023towards}        & 10 & 77.03$\pm$0.19 & 13.08~$(\downarrow7.6\times)$    \\
             & DECOLLE~\cite{kaiser2020synaptic}     & 10 & 61.40$\pm$0.51 & 10.03~$(\downarrow9.9\times)$    \\
             & ELL~\cite{ma2023ell}         & 10 & 72.00$\pm$0.33 & 10.03~$(\downarrow9.9\times)$    \\
             & \baseline{}STDL        & \baseline{}10 & \baseline{}77.87$\pm$0.19 & \baseline{}10.50~$(\downarrow9.4\times)$    \\ \hline
GAIT-DAY-DVS & EV-Gait GCN~\cite{gait_day_pami} & 1  & 89.9       & -        \\
             & TCSA~\cite{yao2023attentionsnn}        & 20 & 92.78$\pm$0.79 & -        \\
             & ASA~\cite{Yao_2023_ICCV}         & 60 & 93.6       & -        \\ \cline{2-5} 
             & BPTT~\cite{wu2018spatio}        & 20 & 94.15$\pm$0.51 & 227.36   \\
             & SLTT~\cite{meng2023towards}        & 20 & 93.88$\pm$0.57 & 15.49~$(\downarrow14.7\times)$    \\
             & DECOLLE~\cite{kaiser2020synaptic}     & 20 & 84.28$\pm$0.88 & 8.89~$(\downarrow25.6\times)$     \\
             & ELL~\cite{ma2023ell}         & 20 & 85.22$\pm$0.87 & 8.89~$(\downarrow25.6\times)$     \\
             & \baseline{}STDL        & \baseline{}20 & \baseline{}94.40$\pm$0.58 & \baseline{}10.91~$(\downarrow20.8\times)$    \\ \midrule
HAR-DVS      & SlowFast~\cite{2024_hardvs}    & -  & 46.54      & -        \\
             & ACTION-Net~\cite{2024_hardvs}  & -  & 46.85      & -        \\
             & ASA~\cite{Yao_2023_ICCV}         & 8  & 47.10      & -        \\ \cline{2-5} 
             & BPTT~\cite{wu2018spatio}        & 4  & 48.64      & 51.20    \\
             & SLTT~\cite{meng2023towards}        & 4  & 48.01      & 15.86~$(\downarrow3.2\times)$    \\
             & DECOLLE~\cite{kaiser2020synaptic}     & 4  & 32.66      & 12.08~$(\downarrow4.2\times)$    \\
             & ELL~\cite{ma2023ell}         & 4  & 40.51      & 12.10~$(\downarrow4.2\times)$    \\
             & \baseline{}STDL        & \baseline{}4  & \baseline{}48.46      & \baseline{}12.87~$(\downarrow4.0\times)$    \\ \bottomrule
\end{tabular}%
\label{tab:event_res}
\end{table}
\linespread{1.0}

\subsubsection{STDL Consistently Achieves Superior Performance}
Table~\ref{tab:event_res} presents the results of the proposed STDL method as well as other baseline methods. STDL consistently achieves comparable or even better accuracy when compared to BPTT, SLTT, and other competitive methods across the three datasets.  
Moreover, STDL offers a significant saving in GPU memory consumption compared to BPTT, with the reduction percentage increasing as the number of time steps increases. For instance, on CIFAR10-DVS with $10$ time steps, STDL achieves a memory reduction of $9.4\times$ compared to BPTT, while on GAIT-DAY-DVS with $20$ time steps, the reduction reaches an impressive $20.8\times$. Additionally, STDL significantly outperforms DECOLLE and ELL in accuracy across these datasets. 
These observations align with those from the image classification experiments in Section~\ref{sec:static_image_exp}, demonstrating the STDL's advantages in high accuracy and memory efficiency. 
The consistent superior performance of our STDL method across various datasets verifies its generalizability and effectiveness in different vision applications.

\subsection{GPU Memory Usage Analysis}
\label{sec:mem_ana}
We compare the GPU memory trace of STDL with that of BPTT using ResNet-18 on CIFAR-10 with 4 time steps.
As shown in Fig.~\ref{fig:memory_trace}, BPTT exhibits a gradual increase in GPU memory usage 
as the network layer and time step grow. This is because BPTT performs gradient computation after the full forward propagation is completed, leading to the accumulation of neuronal states in memory.
In comparison, STDL shows fluctuating GPU memory usage that peaks at a middle layer due to its local updates, which allows cached states to be freed up. Notably, the non-linear increase in GPU memory usage as the network depth grows reflects the decreasing memory footprint of the layers as they approach the output layer in ResNet-18.
This demonstrates the rationality of STDL in combining several successive layers into a single subnetwork without affecting the peak memory footprint. 
The visualization of GPU memory evolution provides a compelling explanation of the memory efficiency advantage of STDL.

\begin{figure}[t]
\centering\includegraphics[width=0.85\linewidth]{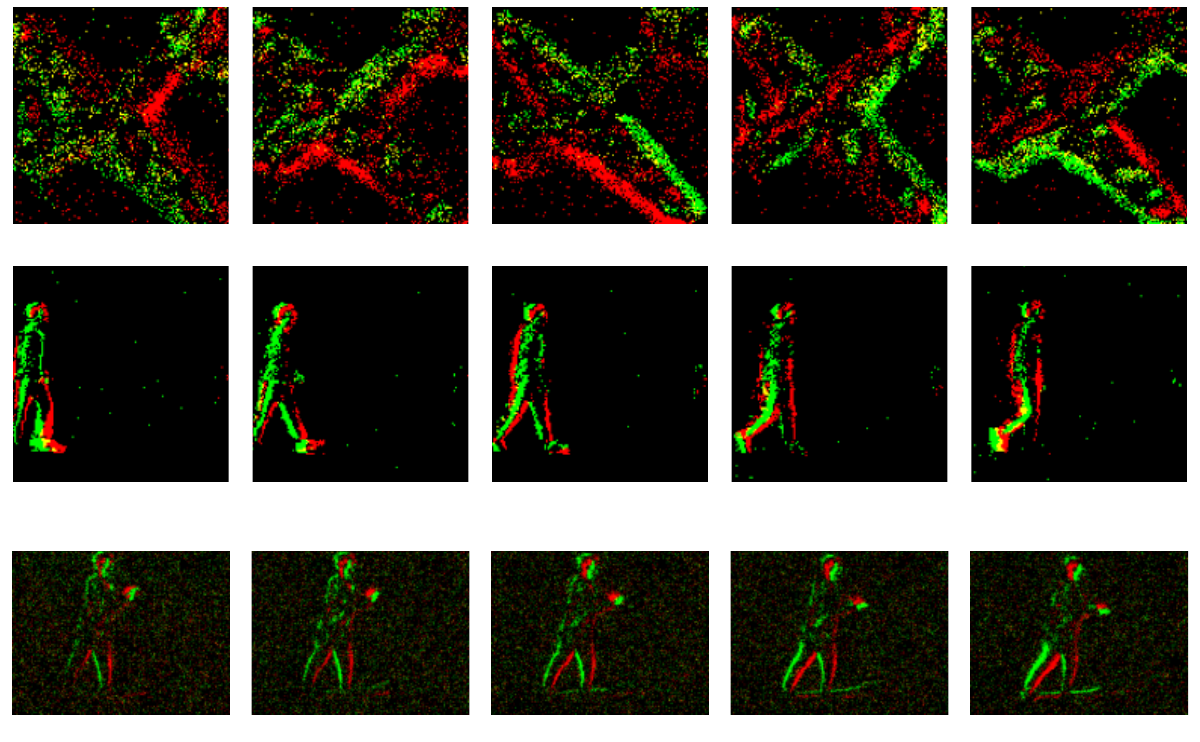}
\caption{\textbf{Visualization of event streams} (accumulated over 10ms) in CIFAR10-DVS~\cite{li2017CIFAR10}, GAIT-DAY-DVS~\cite{gait_day_pami}, and HAR-DVS~\cite{2024_hardvs} from top to bottom. 
	}
\label{fig:dvs_samples}
\end{figure}

\begin{figure}[!t]
\centering\includegraphics[width=0.85\linewidth]{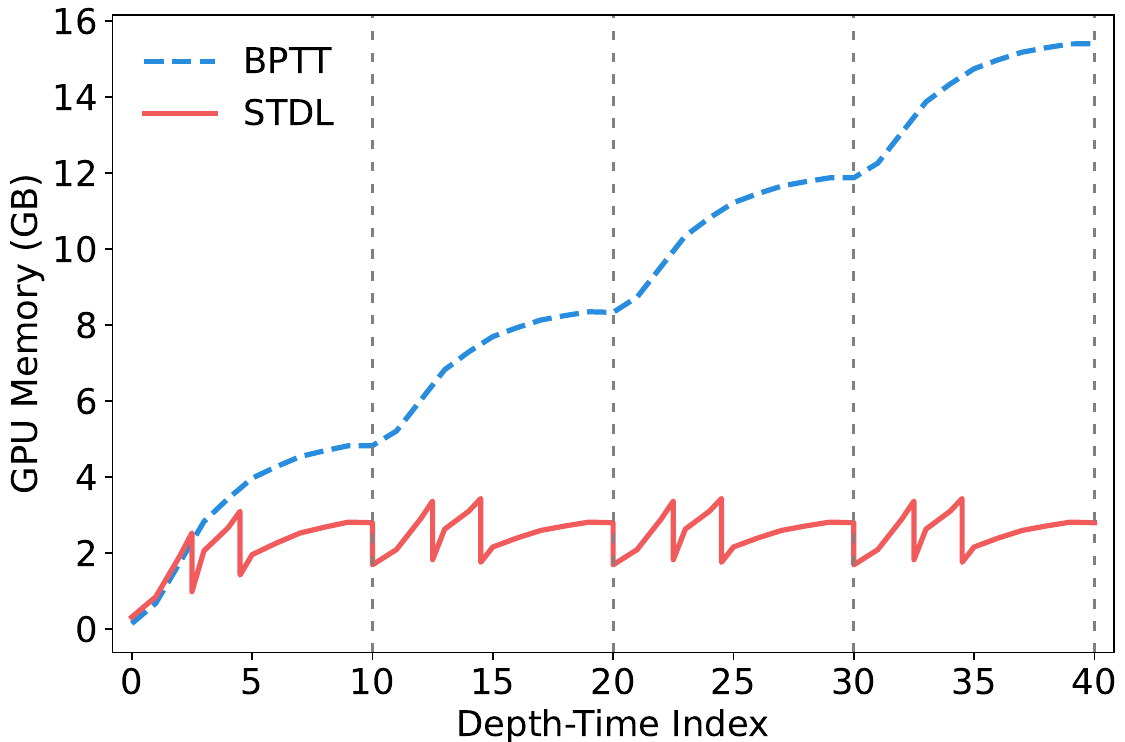}
\caption{\textbf{GPU memory usage pattern with respect to the layer and time index.} ResNet-18 is trained with BPTT and STDL on CIFAR-10 with $4$ steps.
	}
\label{fig:memory_trace}
\end{figure}

\subsection{Representation Similarity Analysis}
\label{sec:rep_ana}
\begin{figure*}[!t]
\centering\includegraphics[width=0.85\linewidth]{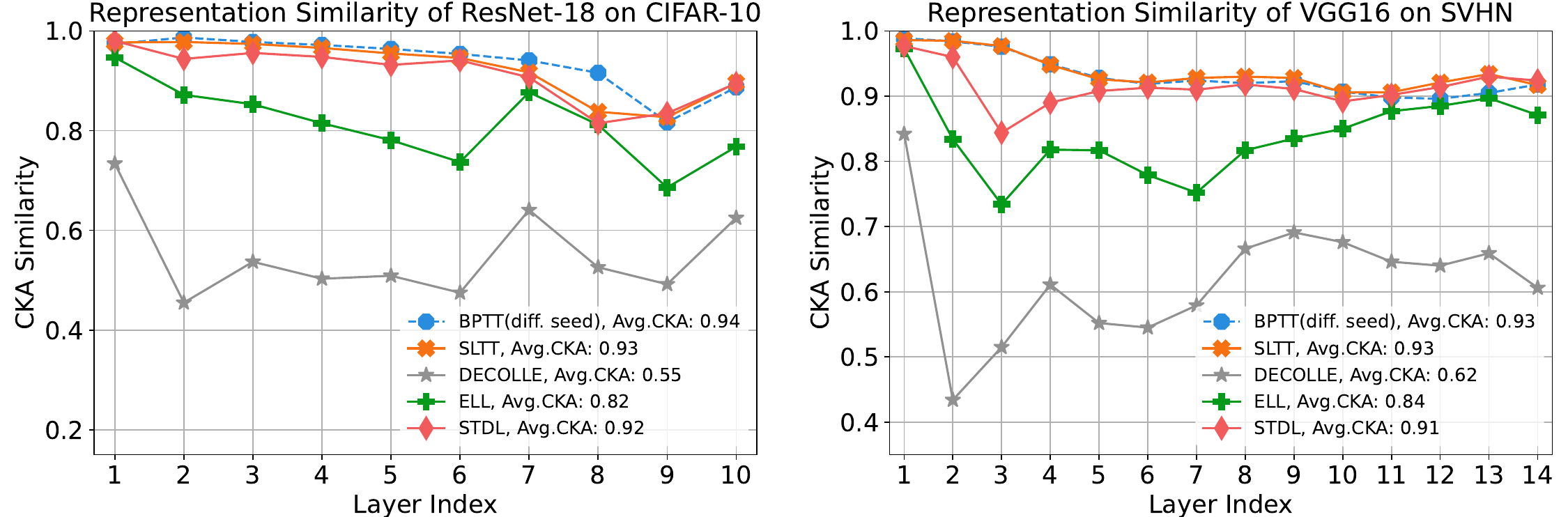}
\caption{\textbf{Comparison of layer-wise representation similarity.} We utilize CKA~\cite{pmlr-v97-kornblith19a} to measure the layer-wise similarity of representations between BPTT and other learning rules. To provide a fair baseline for BPTT, we measure the similarity between two networks trained with different random seeds. The average CKA similarity scores over all layers for different learning rules are provided in the legend.}
\label{fig:cka}
\end{figure*}
\begin{figure*}[!t]
\centering\includegraphics[width=0.85\linewidth]{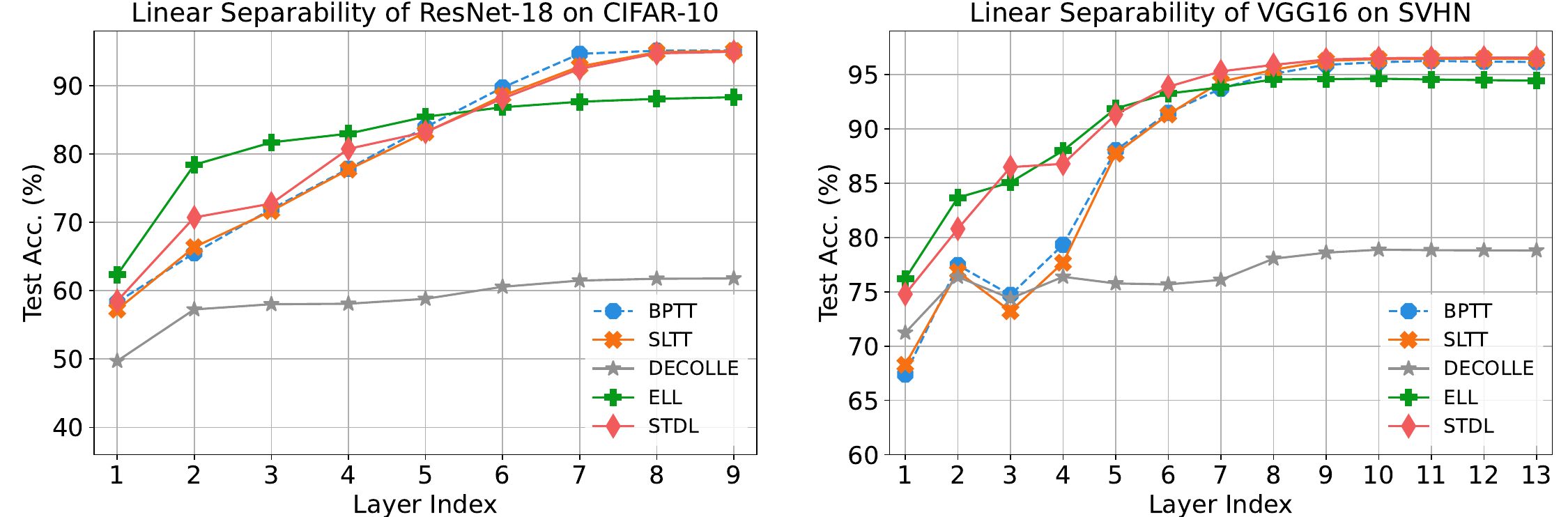}
\caption{\textbf{Comparison of layer-wise linear separability among STDL and baselines.}  
	}
\label{fig:linear}
\end{figure*}

In this part, we analyze the layer-wise representations learned by STDL and compare them with those obtained through BPTT training.
\subsubsection{Representation Similarity}
We use centered kernel alignment (CKA)~\cite{pmlr-v97-kornblith19a} to quantitatively analyze the similarity in the representations of network layers between BPTT and other learning rules. 
The spike-based representation for a layer is obtained by computing its firing rate. We calculate CKA for each layer and an average CKA score across all layers. The results of ResNet-18 and VGG16 are presented in Fig.~\ref{fig:cka}. The STDL-trained layers generate representations that closely resemble those produced by BPTT, while DECOLLE and ELL show significantly lower representation similarity to BPTT. 
This offers a compelling explanation for the efficacy of STDL. 

\subsubsection{Linear Probing}
We further employ the linear probing technique to analyze the layer-wise linear separability of STDL as well as other learning methods.
Specifically, we freeze the parameters of well-trained networks and further train additional linear classifiers that are attached to each hidden layer. The results are provided in Fig.~\ref{fig:linear}.

We observe that STDL closely follows BPTT in terms of layer-wise linear separability, whereas DECOLLE and ELL exhibit drastically different patterns compared to BPTT. Specifically, DECOLLE shows consistently lower linear separability across all layers, indicating that its random and fixed local classifiers are insufficient for guiding hidden layers to produce meaningful representations. 
ELL, on the other hand, encounters the weak coupling issue,  where the early layers exhibit greater linear separability than their BPTT counterparts, but this separability significantly deteriorates in layers closer to the output.
In comparison, STDL initially shows lower linear separability in the early layers, but exhibits substantial improvements in the middle and output layers, closely resembling the pattern observed in BPTT. These findings confirm that STDL effectively addresses the weak coupling issue and achieves layer-wise representations comparable to those trained by BPTT. This alignment underpins the comparable accuracy between STDL and BPTT.

\subsection{Influences of Auxiliary Networks and Subnetworks}
\label{sec:ab_ana}

We conduct ablation studies to examine the individual contributions of the two core components in STDL: the auxiliary network construction strategy and the greedy network partitioning method. All experiments are performed on CIFAR-10 using ResNet-18, and the results are summarized in Fig.~\ref{fig:ablation_study}.

We begin by analyzing the design of auxiliary networks in STDL, which incorporates two key principles: reusing the structure of subsequent layers, and maximizing expressive capacity within the memory budget. To assess the role of structural reuse, we replace the residual blocks in the auxiliary networks with either $3\!\times\!3$ or $1\!\times\!1$ convolutional layers. Both modifications lead to noticeable drops in accuracy, confirming that generic convolutions fail to provide structure priors. This underscores the importance of leveraging the architecture of subsequent layers to guide subnetwork training effectively.

Next, we evaluate the importance of maximizing the expressive capacity of auxiliary networks. We compare against a simplified variant that uses only a two-layer auxiliary network (i.e., a single layer plus the classifier). This setup results in an accuracy drop of approximately 1\%. The performance degrades further when using a linear classifier as the auxiliary network. Another variant that maximizes only the depth also underperforms. These results demonstrate that the expressive capacity of auxiliary networks, in both depth and width, is critical for superior accuracy.

We then evaluate the impact of the greedy network partitioning method. A layer-wise partitioning, where each subnetwork contains only one layer and is trained with a linear classifier, results in a significant accuracy drop to 88.4\%. Even when the same auxiliary networks are preserved, replacing the greedy partitioning with uniform layer-wise partitioning still leads to reduced accuracy. These results underscore the critical role of our network partitioning method in high-accuracy training.

\begin{figure}[!t]
\centering\includegraphics[width=0.8\linewidth]{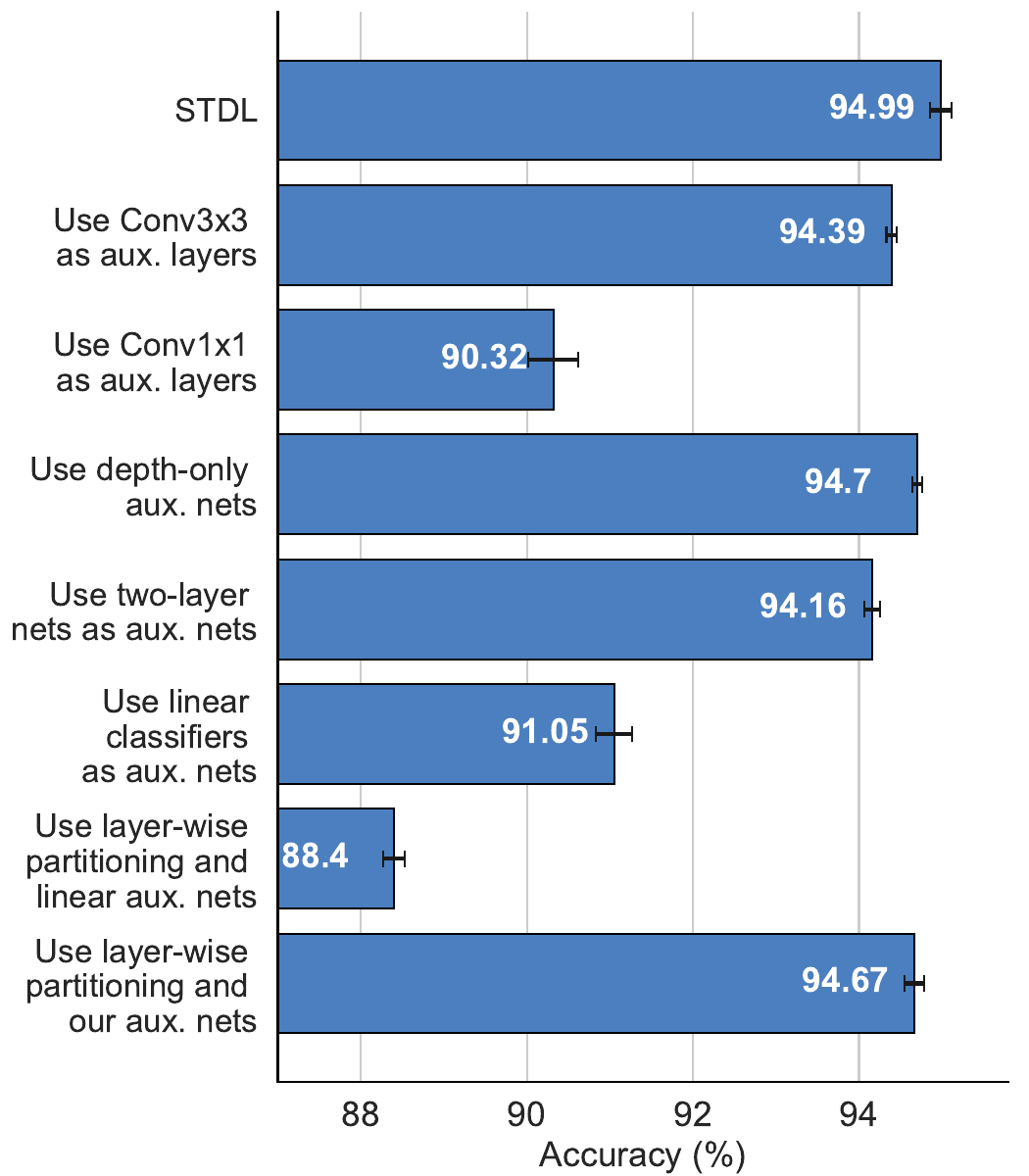}
\caption{\textbf{Influences of auxiliary networks and subnetworks.} Each row denotes a variant of STDL with a modification described on the y-axis. The blue bar chart reports each variant's accuracy for ResNet-18 on CIFAR-10.
}
\label{fig:ablation_study}
\end{figure}

\subsection{Influence of Varying Time Steps}
\label{sec:diff_time_influence}
We investigate the influence of different numbers of time steps on the performance of STDL. We adopt ResNet-18 on CIFAR-10 and vary the number of time steps from $1$ to $6$. We compare STDL against BPTT in terms of accuracy and GPU memory usage. Table~\ref{tab:influence_of_T} shows that STDL consistently achieves comparable accuracy to BPTT across different time steps. Additionally, unlike BPTT, whose GPU memory footprint linearly increases over time, STDL exhibits a consistent memory footprint over time owing to its temporal decoupling. Furthermore, even with just one time step, STDL demands approximately $30$\% less GPU memory than BPTT. This reaffirms that the spatial decoupling in STDL can effectively reduce the memory footprint without sacrificing accuracy.

\begin{table}[!t]
\linespread{1.02}
\centering
\caption{Comparison of accuracy and GPU memory usage between STDL and BPTT with varying time steps on CIFAR-10 using ResNet-18.}
\setlength{\tabcolsep}{1mm}
\begin{tabular}{cccccc}
\toprule
     &           & T=1        & T=2        & T=4        & T=6        \\ \midrule
BPTT & Mem. (GB) & 4.56       & 8.40       & 15.45      & 22.50      \\
     & Acc. (\%) & 94.01$\pm$0.22 & 94.98$\pm$0.17 & 95.12$\pm$0.15 & 95.20$\pm$0.11 \\ \midrule
\baseline{}STDL & \baseline{}Mem. (GB) & \baseline{}3.17       & \baseline{}3.51       & \baseline{}3.51       & \baseline{}3.51       \\
\baseline{}     & \baseline{}Acc. (\%) & \baseline{}93.90$\pm$0.26 & \baseline{}94.94$\pm$0.14 & \baseline{}94.99$\pm$0.13 & \baseline{}95.18$\pm$0.08 \\ \bottomrule
\end{tabular}%
\label{tab:influence_of_T}
\end{table}
\linespread{1.0}

\subsection{Trade-Off between GPU Memory Efficiency and Accuracy}
\label{sec:memory_acc_trade}
\begin{figure}[!htb]
\centering\includegraphics[width=0.85\linewidth]{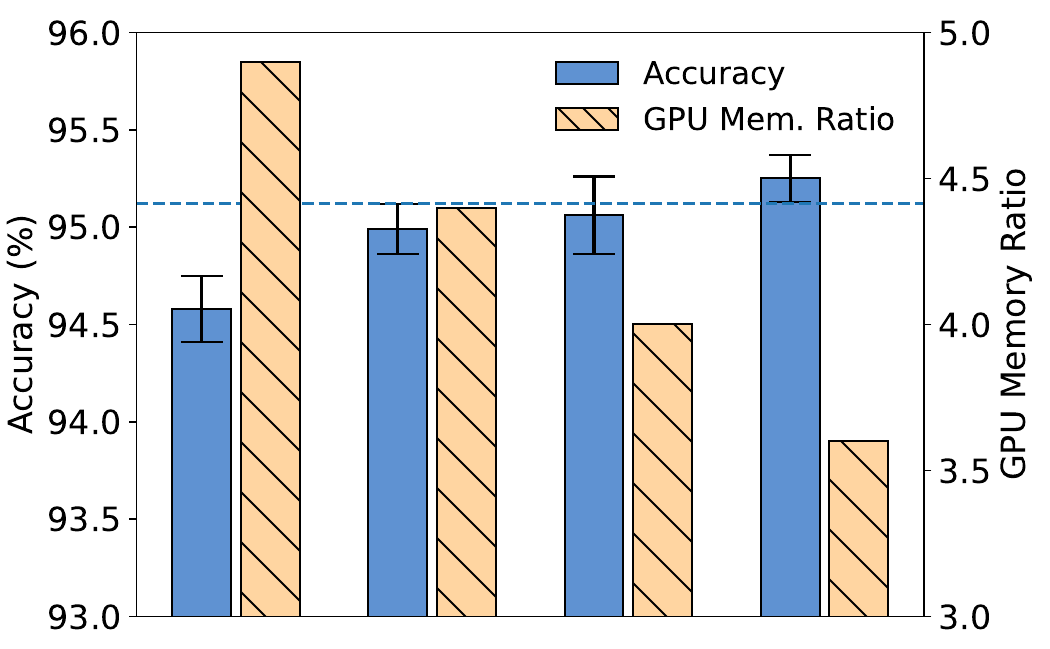}
\caption{\textbf{Trade-offs between GPU memory efficiency and accuracy in STDL.} The dashed line denotes BPTT's accuracy. GPU memory ratio is the memory of BPTT divided by that of STDL. 
}
\label{fig:mem_acc_trade_off}
\end{figure}

\begin{table}[t]
\centering
\caption{Generalization of STDL on different spiking neuron models. The GAIT-DAY-DVS dataset is adopted.}
\linespread{1.02}
\setlength{\tabcolsep}{1.5mm}
\begin{tabular}{cccc}
\toprule
Neuron Model & Method & Accuracy (\%)   & GPU Mem. (GB) \\ \midrule
PLIF~\cite{fang2021incorporating}   & BPTT   & 94.50$\pm$0.36  & 300.92        \\
       & \baseline{}STDL   & \baseline{}94.80$\pm$0.30  & \baseline{}12.47~$(\downarrow24.1\times)$ \\ \midrule
ALIF~\cite{yin2021accurate}   & BPTT   & 94.97$\pm$0.32 & 308.48        \\
       & \baseline{}STDL   & \baseline{}95.02$\pm$0.19 & \baseline{}21.42~$(\downarrow14.4\times)$ \\ \bottomrule
\end{tabular}%
\label{tab:neuron}
\end{table}
\linespread{1.0}
In this part, we investigate the trade-off between GPU memory consumption and model accuracy in STDL. To this end, we vary the level of memory constraint, resulting in different configurations of subnetworks and their corresponding auxiliary networks. Each configuration is then evaluated in terms of accuracy and GPU memory usage, normalized against the memory footprint of BPTT. The results, obtained using ResNet-18 on CIFAR-10, are shown in Fig.~\ref{fig:mem_acc_trade_off}.
 
The results reveal that STDL provides a flexible balance between accuracy and memory efficiency. When the memory constraint is relaxed, STDL yields higher accuracy at the cost of increased memory usage. Conversely, tighter memory constraints lead to improved efficiency with only a minor reduction in accuracy. Notably, STDL achieves substantial memory savings even in high-accuracy regimes. For instance, 
STDL achieves an accuracy of 95.25\%, exceeding that of BPTT, while simultaneously reducing GPU memory footprint by a factor of $3.6\times$. This ability to retain competitive performance under different constrained resources highlights the applicability of STDL in memory-constrained scenarios.

\subsection{Generalization to Different Neuron Models}
\label{sec:diff_neuron_model}
We evaluate the generalization ability of STDL on two widely-used spiking neuron models: parametric LIF~(PLIF)~\cite{fang2021incorporating} and adaptive LIF (ALIF)~\cite{yin2021accurate}. PLIF introduces trainable decay factors, which necessitate the storage of extra neuronal states during forward propagation. 
ALIF additionally includes an internal state that evolves over time, which enables adaptive adjustment of neuronal thresholds but also leads to higher GPU memory consumption.

Table~\ref{tab:neuron} presents the results on the GAIT-DAY-DVS dataset~\cite{gait_day_pami}. Despite the varying memory requirements of these different neuron models, STDL consistently achieves significant memory reduction compared to BPTT, without compromising accuracy. Specifically, STDL yields over $24\times$ and $14\times$ memory savings for the PLIF and ALIF models, respectively. These results underscore the strong generalization ability of STDL across different spiking neuron models.

\section{Conclusion}
\label{sec:conclu}
In this work, we proposed STDL, a novel training method that combines the high memory efficiency of local learning with the high accuracy of BPTT for SNNs. STDL achieves this by partitioning the full network into subnetworks, each trained independently with an auxiliary network in an online manner. To enhance the coupling among the subnetworks, we propose the greedy network partitioning algorithm that minimizes the number of subnetworks under a given memory constraint. In addition, we propose an information-theoretically grounded auxiliary network construction method that encourages alignment with the representations learned by BPTT. Extensive experiments across a range of datasets, architectures, and neuron models demonstrate that STDL consistently achieves accuracy comparable to that of BPTT while significantly reducing GPU memory consumption. This work would offer a promising direction for scalable and memory-efficient training of SNNs and opens up new opportunities for training SNNs on resource-constrained neuromorphic hardware.

\bibliographystyle{IEEEtran}
\bibliography{myRefs}

\clearpage
\onecolumn

\appendices

{\centering\section*{Appendix for\\``Spatio-Temporal Decoupled Learning for Spiking Neural Networks''}}

This appendix presents implementation details, including datasets, training setup, and structures of the subnetworks and the auxiliary networks in STDL.
\section{Datasets}
\label{app:datasets}
Our experiments are based on seven widely used benchmark datasets, including four static image classification datasets (i.e., CIFAR10, CIFAR100, SVHN, and ImageNet) and three event-based vision datasets (i.e., CIFAR10-DVS, GAIT-DAY-DVS, and HAR-DVS)

\begin{itemize}
    \item \textbf{CIFAR10} and \textbf{CIFAR100} datasets are consisted of $50,000$ training and $10,000$ test images, in $10$ classes and $100$ classes, respectively.  We use the standard data augmentation in the training set, where $4$ pixels are padded on each side of the samples followed by a $32\times32$ crop and a random horizontal flip. Following~\cite{deng2021temporal,li2021differentiable,wang2023adaptive}, we also use the autoaugment and cutout techniques for data augmentation.
    \item \textbf{SVHN} dataset contains $73,257$ images for training and $26,032$ images for testing. Training samples are augmented by padding $2$ pixels on each side of images followed by a $32\times32$ crop. 
    \item \textbf{ImageNet} is a $1,000$-class dataset with $1.2$ million images for training and $50,000$ images for validation. The standard data augmentation is used where a $224\times224$ random crop followed by a random horizontal flip is adopted for training samples, and a $224\times224$ central crop is applied for test samples~\cite{he2016deep,fang2021deep}. 
    \item \textbf{CIFAR10-DVS} is derived from the CIFAR-10 dataset by scanning each image with repeated closed-loop movement in front of a DVS camera. It contains $9,000$ samples for training and $1,000$ samples for testing, with a spatial resolution of $128\times128$. Corresponding to the CIFAR-10 classes, CIFAR10-DVS also consists of $10$ classes. We employ the standard preprocessing pipeline in SpikingJelly~\cite{2023spikingjelly} to transform events into frames for further processing. We do not apply any data augmentation techniques.
    \item \textbf{GAIT-DAY-DVS} captures the gait patterns of $20$ volunteers under daylight conditions using a DVS128 camera. The dataset comprises $2,000$ training samples and $2,000$ testing samples, each with a resolution of $128\times128$. It encompasses a total of $20$ distinct gait classes. We use SpikingJelly~\cite{2023spikingjelly} for preprocessing and do not employ any data augmentation techniques.
    \item \textbf{HAR-DVS} is an event-based human activity recognition~(HAR) dataset recorded by a DAVIS346 camera with a spatial resolution of $346\times260$. It is currently the largest event-based HAR dataset, containing $300$ activity classes and a total of $107,646$ samples. Events are converted to frames with SpikingJelly~\cite{2023spikingjelly}. The standard data augmentation is used where a $224\times224$ random crop followed by a random horizontal flip is applied for training samples, and a $224\times224$ central crop is used for test samples.
\end{itemize}

\section{Training Setup}
\label{app:train_setup}
Our training settings are summarized in Table \ref{tab:train_details}. We keep them consistent across all baseline learning methods for fair comparisons. The only exception is with BPTT on ImageNet, where we encounter memory constraints. To mitigate this issue, we adjust the batch size to $256$ and scale the learning rate accordingly.

\begin{table}[!htb]
\caption{Training configurations and hyper-parameters for all evaluated datasets.}
\label{tab:train_details}
\resizebox{\textwidth}{!}{%
\begin{tabular}{cccccccccc}
\hline
Dataset &
  Epochs &
  Optimizer &
  \begin{tabular}[c]{@{}c@{}}Learning \\ Rate\end{tabular} &
  \begin{tabular}[c]{@{}c@{}}Learning Rate \\ Schedule\end{tabular} &
  \begin{tabular}[c]{@{}c@{}}Batch \\ Size\end{tabular} &
  \begin{tabular}[c]{@{}c@{}}Weight \\ Decay\end{tabular} &
  \begin{tabular}[c]{@{}c@{}}Neuronal\\  Decay\end{tabular} &
  Threshold &
  \begin{tabular}[c]{@{}c@{}}\# Time \\ Steps (T)\end{tabular} \\ \hline
CIFAR10            & 400 & SGD   & 0.4  & Cosine annealing & 512 & 0.00005 & 0.1 & 1.0 & 4  \\
CIFAR100           & 400 & SGD   & 0.4  & Cosine annealing & 512 & 0.00005 & 0.1 & 1.0 & 4  \\
SVHN               & 200 & SGD   & 0.1  & Cosine annealing & 512 & 0.00005 & 0.1 & 1.0 & 4  \\
ImageNet (pretrain) & 100 & SGD   & 0.15 & Cosine annealing & 300 & 0.00001 & 0.2 & 1.0 & 1  \\
ImageNet (finetune) & 10  & SGD   & 0.002 & Cosine annealing & 512 & 0    & 0.2 & 1.0 & 4  \\
CIFAR10-DVS        & 300 & AdamW & 0.001 & Cosine annealing & 100 & 0.0005 & 0.1 & 1.0 & 10 \\
GAIT-DAY-DVS       & 100 & AdamW & 0.0005 & Cosine annealing & 100 & 0.0005 & 0.1 & 1.0 & 20 \\
HAR-DVS            & 100 & AdamW  & 0.001 & Cosine annealing & 100 & 0.0005 & 0.1 & 1.0 & 4  \\ \hline
\end{tabular}%
}
\end{table}

\section{Details of Subnetworks and Auxiliary Networks}
\label{app:subn_aux_nets}
Table~\ref{tab:aux_details} presents the details of subnetworks and auxiliary networks constructed using our STDL method for all spiking networks employed in our experiments. For clarity, the following notations are used: R represents a residual block, C denotes a convolutional layer, AP signifies average pooling, and FC indicates a fully connected layer. Additionally, C3 refers to a 3×3 convolutional kernel size. The numerical value preceding C, R, and FC indicates the number of output channels. Furthermore, k, s, and p in brackets denote the kernel size, stride, and padding, respectively. Table~\ref{tab:aux_details} also includes the given memory efficiency $\rho$, which is computed as the ratio of the GPU memory allocated to the online learning of each subnetwork and its auxiliary network compared to that of the online learning of the original network. 

\begin{table}[!htb]
\caption{Structures of subnetworks and auxiliary networks for all adopted spiking networks. $\beta$ represents the maximum GPU memory ratio between each local module and all primary network layers.}
\label{tab:aux_details}
\resizebox{\textwidth}{!}{%
\begin{tabular}{c|c|cc|cc|c}
\hline
\multirow{2}{*}{Network} &
  \multirow{2}{*}{$\rho$} &
  \multicolumn{2}{c|}{1st Local Module} &
  \multicolumn{2}{c|}{2nd Local Module} &
  3rd Local Module \\
 &
   &
  Subnetwork &
  Auxiliary Network &
  SubNetwork &
  Auxiliary Network &
  Subnetwork \\ \hline
ResNet-18 &
  0.7 &
  64C3-64R- &
  \begin{tabular}[c]{@{}c@{}}AP(k2s2)-256R(s2)\\ -512R(s2)-512R-10FC\end{tabular} &
  64R-128R(s2)- &
  256R(s2)-512R(s2)-10FC &
  \begin{tabular}[c]{@{}c@{}}128R-256R(s2)-256R\\ -512R(s2)-512R-10FC\end{tabular} \\ \hline
ResNet-19 &
  0.6 &
  128C3-128R- &
  256R(s2)-512R(s2)-100FC &
  128R-128R- &
  AP(k2s2)-512R(s2)-100FC &
  \begin{tabular}[c]{@{}c@{}}256R(s2)-256R-256R\\ -512R(s2)-512R-100FC\end{tabular} \\ \hline
VGG16 &
  0.7 &
  64C3- &
  AP(k16s16)-512C3-10FC &
  64C3-AP(k2s2)-128C3- &
  AP(k8s8)-512C3-10FC &
  \begin{tabular}[c]{@{}c@{}}128C3-AP(k2s2)-256C3-256C3\\ -256C3-AP(k2s2)-512C3-512C3-512C3\\ -AP(k2s2)-512C3-512C3-512C3-10FC\end{tabular} \\ \hline
SEWResNet-34 &
  0.7 &
  \begin{tabular}[c]{@{}c@{}}64C3-AP(k3s2p1)-64SEWR\\ -64SEWR-64SEWR-\end{tabular} &
  \begin{tabular}[c]{@{}c@{}}AP(k4s4)-512SEWR(s2)\\ -512SEWR(s1)-\\ 512SEWR(s1)-1000FC\end{tabular} &
  \multicolumn{2}{c|}{\begin{tabular}[c]{@{}c@{}}-128SEWR(s2)-128SEWR-128SEWR-128SEWR-256SEWR(s2)\\ -256SEWR-256SEWR-256SEWR-256SEWR-256SEWR\\ -512SEWR(s2)-512SEWR(s2)-512SEWR(s2)-1000FC\end{tabular}} &
   \\ \hline
\end{tabular}%
}
\end{table}

\section{Implementation Details in Section~\ref{subsec:understanding}}
\label{app:imp_incomp}
We provide implementation details for the experiments in Section~\ref{subsec:understanding}. 
We use supervised local learning to train ResNet-18 on CIFAR10, which is evenly decomposed into $\{1, 3, 5, 7, 9\}$ subnetworks, each independently trained with a trainable linear classifier~\cite{ma2023ell} under a consistent training setup. Specifically, we utilize the SGD optimizer and set the weight decay to $0.00005$, the momentum to $0.9$, and the initial learning rate to $0.4$. To schedule the learning rate, we employ the cosine annealing technique. We train the networks for $400$ epochs with a batch size of $512$. In terms of neuronal parameters, we set the decay factor to $0.1$ and the threshold to $1.0$. The total number of time steps is set to $1$ to isolate the effect of spatial decoupling.



 




\vfill

\end{document}